\documentclass[twoside]{article}

\usepackage[accepted]{aistats2022}
\usepackage[round]{natbib}
\renewcommand{\bibname}{References}
\renewcommand{\bibsection}{\subsubsection*{\bibname}}

\usepackage{epsfig}
\usepackage{epstopdf}
\usepackage{hyperref}
\usepackage{smile}
\usepackage{wrapfig,lipsum}
\usepackage{tabularx}
\usepackage{booktabs}
\usepackage{amsfonts}
\usepackage{amsmath}
\usepackage{amssymb}
\usepackage{bm}
\usepackage{algorithm}
\usepackage{algorithmic}
\usepackage{amsthm}
\usepackage{enumerate}
\usepackage{enumitem}

\def \EE {\mathbb{E}}

\def \RR {\mathbb{R}}
\def \diag {\mathrm{diag}}
\def \xb {\mathbf{x}}
\def \yb {\mathbf{y}}
\def \zb {\mathbf{z}}
\def \cb {\mathbf{c}}

\def \eb {\mathbf{e}}

\def \mb {\mathbf{m}}

\def \gb {\mathbf{g}}

\def \gbt {\Tilde{\mathbf{g}}}
\def \gbh {\hat{\mathbf{g}}}
\def \vb {\mathbf{v}}

\def \bdelta {\bm{\delta}}

\def \Vb {\mathbf{V}}
\def \Vbh {\hat{\mathbf{V}}}
\def \Ib {\mathbf{I}}

\usepackage{color}
\usepackage{xcolor}
\definecolor{carnelian}{rgb}{0.7, 0.11, 0.11}
\newcommand{\ifcomments}{\iftrue}

\begin{document}

\twocolumn[

\aistatstitle{Communication-Compressed Adaptive Gradient Method for Distributed Nonconvex Optimization}

\aistatsauthor{ Yujia Wang \And Lu Lin \And  Jinghui Chen }

\aistatsaddress{ Pennsylvania State University \And  University of Virginia \And Pennsylvania State University } ]

\begin{abstract}
Due to the explosion in the size of the training datasets, distributed learning has received growing interest in recent years. One of the major bottlenecks is the large communication cost between the central server and the local workers. While error feedback compression has been proven to be successful in reducing communication costs with stochastic gradient descent (SGD), there are much fewer attempts in building  communication-efficient adaptive gradient methods with provable guarantees, which are widely used in training large-scale machine learning models. In this paper, we propose a new communication-compressed AMSGrad for distributed nonconvex optimization problem, which is provably efficient. Our proposed distributed learning framework features an effective gradient compression strategy and a worker-side model update design. We prove that the proposed communication-efficient distributed adaptive gradient method converges to the first-order stationary point with the same iteration complexity as uncompressed vanilla AMSGrad in the stochastic nonconvex optimization setting. Experiments on various benchmarks back up our theory.
\end{abstract}

\section{Introduction}
The recent success of deep learning and large-scale training made it possible to use machine learning to solve complicated real-world tasks, such as computer vision \citep{he2016deep}, speech recognition \citep{hinton2012deep}, natural language processing \citep{devlin2018bert}, etc. Due to the explosion in the size of training datasets in recent years, single GPU training on such models can easily take up to weeks or even months to finish. As a consequence, distributed training algorithms have attracted growing interest over the years. 
In standard distributed settings with one central server and $n$ workers, local workers parallelly compute local gradients, while the server aggregates the gradients from the workers, updates the model, and sends the new model back to the workers. However, data transmissions across the machines can be quite expansive in terms of both communication costs (especially for cellular networks) and latency. Therefore, various methods were studied in order to achieve communication-efficient distributed learning by either reducing communication bits \citep{efsparse,basu2019qsparse,karimireddy2019error} or saving the number of communication rounds \citep{zheng2017asynchronous,chen2021cada}.
One of the principled approaches for effectively reducing the communication cost is to perform gradient compression before transmissions. It directly compresses the local fresh gradient on each worker before uploading the gradient to the server. However, the compression would slow down the convergence or even diverge \citep{beznosikov2020biased} due to the loss of information at each compression step. Later on, error feedback \citep{efsparse,karimireddy2019error} was proposed to alleviate this problem and reduce the information loss by proposing a compensating error sequence. Each worker compresses and uploads the combination of the last step compensating error and the local fresh gradient instead of compressing the gradient directly. Recent studies \citep{1bitsgd, karimireddy2019error, efsparse} show that error feedback has been widely used in distributed SGD to communicate efficiently and ensure the same convergence rate as vanilla SGD. 

While communication-efficient distributed SGD has been widely studied, there are much fewer attempts in building communication-efficient distributed adaptive gradient methods \citep{kingma2014adam,j.2018on}. It has been shown that SGD works less efficiently compared with adaptive gradient methods when training large-scale models such as BERT \citep{devlin2018bert}, GPT-3 \citep{brown2020language} and GAN \citep{goodfellow2014generative}. One of the major challenges is that the traditional error feedback mechanism is not compatible with adaptive gradient methods since the variance term (i.e., second-order moments of the historical gradients) in the adapted gradient method can be unstable due to the accumulated compression error \citep{tang20211}. 1-bit Adam \citep{tang20211} partially solve this problem by first using vanilla Adam at the beginning of training to get an estimate of variance term, then freezing this variance term and performing distributed Adam with the fixed variance. While this solution indeed gets rid of the unstable variance issue, the adaptivity is no longer changing and the major part of the algorithm is actually more similar to SGD with momentum. Due to the same reason, the convergence guarantee provided in \cite{tang20211} is not related to distributed Adam but is an extension to distributed SGD with momentum.

In this paper, we develop a new compression method which is provably efficient, namely \textbf{C}ommunication-compressed \textbf{D}istributed \textbf{Ada}ptive gradient \textbf{M}ethod (CD-Adam). We revisit the unstable variance issue in distributed AMSGrad \citep{j.2018on} and identify the need for an improved gradient compression strategy. Specifically, we adopt the Markov compression sequence recently proposed by \citet{richtarik2021ef21}, which ideally could lead to contractive compression error for gradient descent if the gradient sequence (to be compressed) is convergent. Note that this is not simply applying the conclusion of \citet{richtarik2021ef21} to the distributed AMSGrad setting, as \citet{richtarik2021ef21} only deals with worker-to-server compression. While extending \citet{richtarik2021ef21} to both worker-to-server and server-to-worker compression may seem straightforward for standard gradient descent, it is particularly tricky for the adaptive gradient method (see Section \ref{sec:cdadam} for details). Therefore, it requires carefully designed algorithms and analyses to build a provably efficient distributed adaptive gradient method.

We summarize our contribution as follows:
\begin{itemize}
    \item We propose a new communication-compressed distributed AMSGrad approach which solves the bottleneck of applying communication compression strategies for fully-functional\footnote{Here we use ``fully-functional'' to differentiate with the variance-freezed Adam used in \citet{tang20211}.} adaptive gradient methods in the distributed setting. The proposed method largely reduces the communication cost by enforcing both worker-to-server and server-to-worker communication compression without any warm-up stages.
    \item We theoretically prove the convergence of our proposed algorithm in the nonconvex stochastic optimization setting. Specifically, we show that our proposed fully compressed  CD-Adam can reach its first-order $\epsilon$-stationary point within $ O(1/\epsilon^2)$ iterations, which is the same iteration complexity as the uncompressed vanilla AMSGrad.
    This suggests that without applying any variance-freezing tricks, the fully compressed distributed adaptive gradient method can still provably converge to its $\epsilon$-stationary point.
    \item Thorough experiments on various real-world benchmarks show that our proposed CD-Adam reduces the communication cost by around $32\times$ over the original AMSGrad and around $5\times$ over 1-bit Adam.
\end{itemize}

\section{Related Work}
\noindent\textbf{Stochastic gradient descent and adaptive gradient methods:}
Stochastic gradient descent (SGD) \citep{sgd1951} is broadly used in training large-scale machine learning problems. Despite being simple to implement, SGD can be sensitive to parameters such as learning rate and slow to converge. Adaptive gradient methods were proposed to further improve over SGD. Adam \citep{kingma2014adam}, one of the most popular adaptive gradient methods, has shown to be fast convergent and also robust to hyper-parameters like learning rate. It designs an adaptive learning rate for each different coordinate using the past gradient. Aside from Adam, AdaGrad \citep{Duchi2011AdaptiveSM} applied the second moment of the gradient to adaptive the learning rate, RMSProp \citep{tieleman2012lecture} further used the second moment of the gradient with a decay rate, AdaDelta \citep{zeiler2012adadelta} is an extension of AdaGrad with a non-increasing learning rate. Later on, \citet{j.2018on} pointed out a non-convergence issue of Adam, and proposed a new AMSGrad algorithm for ensuring the convergence.
\citet{zaheer2018adaptive} studied the effect of adaptive denominator constant $\nu$ and mini-batch size in the convergence of adaptive gradient methods. 
AdaBound \citep{luo2019adaptive} proposed with both upper and lower bound for the variance term of Adam.  AdamW \citep{loshchilov2018decoupled} proposed to fix the weight decay regularization in Adam by decoupling the weight decay from the gradient update.
\citet{chen2020closing} proposed a partially adaptive gradient method and proved its convergence in nonconvex settings.
\citet{chen2018on,1808.05671} showed the convergence rate of a class of adaptive gradient methods under the nonconvex stochastic optimization setting.
\citet{alacaoglu2020new} proposed a new framework to derive data-dependent regret bounds with a constant momentum parameter in various settings.

\noindent\textbf{Distributed SGD and error feedback:}
For communication-efficient distributed SGD, one of the most common strategies is to compress the gradients before uploading. \citet{alistarh2017qsgd} provided a theoretical analysis of the centralized compressed distributed SGD. \citet{1bitsgd} compressed the coordinates of the gradient into $\pm 1$ by its sign. \citet{bernstein2018signsgd} proposed signSGD and proved its convergence in the nonconvex setting. Variant works has applied kinds of compression methods such as sparsification \citep{efsparse,basu2019qsparse}, quantization \citep{karimireddy2019error}, and sketching \citep{ivkin2019communication}. \citet{mishchenko2019distributed} proposed DIANA, which adopts the compression of gradient difference in distributed settings.  \citet{horvath2019stochastic} introduced several variants of DIANA including the variance reduction version. \citet{philippenko2020artemis} introduced a bi-directional compression schemes with memory in distributed or federated setting.
Error feedback (or error compensation) largely improves the compression error bound and has been shown to be critical for ensuring fast convergence of the compression mentioned above. \citet{1bitsgd} showed that with  error feedback, even 1-bit gradient communication under SGD still obtains the convergence rate of vanilla SGD. \citet{karimireddy2019error} applied error feedback to signSGD \citep{bernstein2018signsgd} under nonconvex setting, \citet{efsparse} performed error feedback in sparsity strong convex settings, and \citet{stich2019error} proposed a error-feedback framework also works on nonconvex but single node setting. \citet{doublesqueeze} proposed bi-directional SGD method combining with error compensation. \citet{gorbunov2020linearly} propsed a framework with variants of quantized SGD including EC-SGD and EC-SGD-DIANA. \citet{liu2020double,zheng2019communication,philippenko2021preserved} also studied combining bi-directional compression and error feedback strategy. There are also more works on the distributed SGD under nonconvex optimization \citep{koloskova2019decentralized,basu2019qsparse}.

\noindent\textbf{Communication-efficient distributed adaptive gradient method:}
There are only fewer attempts in developing communication-efficient distributed adaptive gradient methods. 1-bit Adam \citep{tang20211} adopted a variance-freezed Adam by pointing out that the variance term of Adam becomes stable in later training stages. Combined with error feedback, 1-bit Adam achieves the same convergence rate as distributed SGD. \citet{zhong2021compressed} proposed a gradient compressed version of LANS \citep{zheng2020accelerated}. \citet{10.1145/3470890} developed a distributed quantized Adam with error feedback. However, the proposed algorithm in \citet{10.1145/3470890} can only converge to its $\epsilon$-stationary point with worker-to-server compression alone but not fully compressed Adam with both worker-to-server and server-to-worker compression. Another related work is CADA \citep{chen2021cada}, which reduced the communication rounds instead of performing communication compression. CADA adaptively reused the stale update parameters, and it achieved the same convergence rate as vanilla AMSGrad for nonconvex optimization. Note that CADA's strategy of reducing communication rounds is orthogonal to ours can possibly be combined with ours for further improvements.

\section{Problem Formulation}
In this paper, we aim to solve the following distributed optimization problem under nonconvex stochastic optimization setting: 
\begin{align} \label{lossfunc}
\min_{\xb \in \RR^d} f(\xb): = \frac{1}{n} \sum_{i=1}^n \underbrace{\EE_{\xi^{(i)}} f_i(\xb;\xi^{(i)})}_{:=f_i(\xb)} ,
\end{align}
where $d$ denotes the dimension of the model, $n$ denotes the total amount of workers, $\xi^{(i)}$ denotes the stochastic noise variable, and $f_i(\xb)$ denotes the loss function on the $i$-th worker. In the stochastic setting, we cannot obtain the full gradient of loss function $f_i(\xb)$. Instead, we can only obtain the unbiased estimators of $\nabla f_i(\xb)$, i.e., $\nabla f_i(\xb;\xi^{(i)})$.

First, let us revisit the vanilla distributed setting of AMSGrad \citep{j.2018on}. Consider a distributed learning system that contains a parameter server and $n$ workers, with each worker $i$ owns its local data from distribution $\cD_i$. Let $\xb_t$ denotes the current model at $t$-th iteration. The server will first broadcast $\xb_t$ to all the workers in each iteration. Each worker $i$ computes the stochastic gradient $\gb_t^{(i)} = \nabla f_i(\xb_t; \xi_t^{(i)})$ using the local samples, and then uploads $\gb_t^{(i)}$ to the server. The server then aggregates the stochastic gradients and obtains ${\gb}_t=\frac{1}{n} \sum_{i=1}^n \gb_t^{(i)}$. The model uses ${\gb}_t$ to update parameter via vanilla single node AMSGrad:
\begin{align*}
 &\mb_t=\beta_1\mb_{t-1}+(1-\beta_1){\gb}_t,   \vb_t=\beta_2\vb_{t-1}+(1-\beta_2){\gb}_t^2,\\ &\hat{\vb}_t=\max(\hat{\vb}_{t-1},\vb_t),   \xb_{t+1}=\xb_t-\alpha_t\Vbh_t^{-1/2}\mb_t,
\end{align*}
where $\hat{\Vb}_t=\diag(\hat{\vb}_t + \nu)$, $\alpha_t >0$ is the step size, $\nu>0$ is a small number for numerical stability, and $\beta_1, \beta_2 \in [0,1]$ are adaptive hyperparameters. 
This vanilla distributed AMSGrad achieves the same convergence rate as its centralized version \citep{j.2018on}. However, the worker-to-server and server-to-worker communications in each iteration can be extremely expansive especially for cellular networks, and the communication latency also makes the overall system less efficient. Therefore, attempts have been made to find novel approaches that can further reduce the communication cost while maintaining a similar convergence rate as its uncompressed counterpart. In the following, let's first revisit some traditional communication compression strategies for distributed SGD and drawbacks of applying them on distributed adaptive gradient methods.

\section{Existing Solutions and Drawbacks}
\label{sec:ef}
\textbf{Naive compression for SGD:} The simplest strategy to reduce communication cost is to directly compress the local gradient $\gb_t^{(i)}$ with a compressor $\cC(\cdot)$ before sending to the server. The server aggregates the compressed gradient $\gbh_t^{(i)} = \cC(\gb_t^{(i)})$ and updates model by
\begin{align*}
    \xb_{t+1} =  \xb_t - \frac{\alpha}{n} \sum_{i=1}^{n} \gbh_t^{(i)},
\end{align*}
where $\alpha$ denotes the learning rate. The common choice of $\cC(\cdot)$ can be top-$k$ \citep{basu2019qsparse} or sign operation \footnote{See the Supplemental \ref{appendix} for more details about the compressors.} (leads to signSGD \citep{bernstein2018signsgd}). Although this naive compression method is intuitive, it can diverge in practice, even in simple quadratic problems \citep{beznosikov2020biased} or constraint linear problems \citep{karimireddy2019error}. Intuitively speaking, one of the major drawbacks of naive compression is that the compression error is accumulating during the training process. Each step will introduce new errors that cannot be canceled later, and the accumulation of compression error leads the divergence. 

\textbf{Error feedback for SGD:} 
Error feedback (EF), or error compensation \citep{karimireddy2019error, efsparse} is widely used for correcting the bias generated by compression errors. Distributed SGD with error feedback effectively reduces the communication bits by introducing a compensating error sequence to cancel the compression error in previous iterations and obtains the same convergence rate as vanilla SGD. Specifically, error feedback introduces a new sequence $\bdelta_{t}^{(i)}$, which denotes the accumulated compression error at iteration $t$. At $t$-th iteration, the $i$-th worker computes the compressed gradient $\gbh_t^{(i)}$ based on the previous iteration's error $\bdelta_{t-1}^{(i)}$ and the current local gradient $\gb_t^{(i)}$, i.e., $\gbh_t^{(i)} = \cC(\gb_t^{(i)}+\bdelta_{t-1}^{(i)})$. And the new compensating error is updated by $\bdelta_{t}^{(i)}=\gb_t^{(i)}+\bdelta_{t-1}^{(i)}-\gbh_t^{(i)}$. Upon compression, the server collects the compressed gradients $\gbh_t^{(i)}$ from all workers and update model parameters via vanilla SGD.  \citet{karimireddy2019error} showed that the compression error of error feedback is bounded by constant if the following assumption hold for the biased compressor.  
\begin{assumption}[Biased compressor] \label{as:compressor}
Consider a biased operator $\cC: \RR^d \rightarrow \RR^d$, there exists a constant $0<\pi\leq 1$ such that
\begin{align}\label{operator}
\EE_\cC[\|\cC(\xb)-\xb\|^2_2]\leq \pi\|\xb\|^2_2, \quad \forall \xb \in \RR ^d.
\end{align}
\end{assumption}
Note that $\pi = 0$ leads to $\cC(\xb)=\xb$. Assumption \ref{as:compressor} is a common assumption for biased compressor \citep{richtarik2021ef21,karimireddy2019error,efsparse}. Canonical examples of the compressor satisfying Assumption \ref{as:compressor} include
top-$k$ or random-$k$ as well as scaled sign compressor\footnote{See the supplemental materials for more details about the above three compressors as well as other compressors satisfying Assumption \ref{as:compressor}.}. With Assumption \ref{as:compressor}, the distributed SGD with error feedback \citep{karimireddy2019error} in nonconvex setting achieves the same convergence rate as vanilla SGD.

\textbf{Drawbacks of error feedback on adaptive gradient methods:}
While error feedback guarantees bounded accumulated gradient compression error, in adaptive gradient methods \citep{kingma2014adam, j.2018on}, the variance term $\vb_t$, which is the moving average of the quadratic of gradient, can be unstable due to accumulating gradient compression error \citep{tang20211}. Specifically, let's  denote $\gb_t$ as the averaged fresh gradient without compression (average of all $\gb_t^{(i)}$), denote $\gbh_t$ as the averaged compressed gradient (average of all $\gbh_t^{(i)}$). The updating rule for $\vb_{t+1}$ follows:
\begin{align}\label{eq:variance}
\vb_{t+1} & = \beta_2\vb_t+(1-\beta_2)\gbh_t^2 \notag\\
& = \beta_2\vb_t+ (1-\beta_2)[\gbh_t-\gb_t+\gb_t]^2 \notag\\
& = \beta_2\vb_t+ (1-\beta_2)\gb_t^2+ \underbrace{(1-\beta_2)(\gbh_t-\gb_t)^2}_{\text{accumulating error term}} \notag\\
& \quad +2(1-\beta_2)\langle \gb_t,\gbh_t-\gb_t\rangle.
\end{align}
\citet{tang20211} claims that the inner product term in \eqref{eq:variance} can possibly be canceled out during training, while the quadratic term will certainly accumulate. Since the traditional error feedback can only guarantee constant compression error, the accumulation of the quadratic error will make the variance diverge. Therefore, the communication-efficient adaptive gradient method actually requires a stronger gradient compression error bound to obtain a stable variance term.
\section{Proposed Method}\label{sec:cdadam}
In this section, we formally develop our proposed \textbf{C}ommunication-compressed \textbf{D}istributed \textbf{Ada}ptive gradient \textbf{M}ethod (CD-Adam). Our proposed method consists of two key components: Markov compression sequence and worker-side model update design, which jointly provide a better compression error bound. We first investigate the definition and property of Markov compression sequence \citep{richtarik2021ef21}.

\textbf{Markov compression sequence:} Markov compression sequence is introduced{\footnote{Note that \citet{kunstner2017fully} also proposed QDGD-F algorithm which is similar to the Markov compression sequence.}} in \citet{richtarik2021ef21}.
Given a biased compressor $\cC(\cdot)$ and a sequence of vectors $\{\wb_t\}$, Markov compression sequence $\{\hat\wb_t\}$ can be recursively defined as: $\hat\wb_0 = \cC(\wb_0),\quad \hat\wb_{t+1}=\hat\wb_t+\cC(\wb_{t+1}-\hat\wb_t)$. The main advantage of Markov compression sequence lies in that the compression error can be largely reduced if the underlying sequence $\{\wb_t\}$ is convergent:
\begin{align}\label{eq:markov}
\big\|\hat\wb_{t+1}-\wb_{t+1}\big\|^2  &= \big\|\hat\wb_{t}+\cC(\wb_{t+1}-\hat\wb_t)-\wb_{t+1}\big\|^2 \notag\\
& \leq \pi \big\|\wb_{t+1}-\hat\wb_t\|^2 \notag \\
& \leq \pi(1+\gamma)\big\|\hat\wb_t-\wb_{t}\|^2 \\ & \qquad +\pi(1+\gamma^{-1})\big\|\wb_{t+1}-\wb_t\big\|^2,\notag 
\end{align}
where the last inequality holds due to Young's inequality. As can be seen from \eqref{eq:markov}, the compression error $\big\|\hat\wb_{t+1}-\wb_{t+1}\big\|$ is directly controlled by   $\|\wb_{t+1}-\wb_t\|$. If the sequence of $\{\wb_t\}$ is convergent, Markov compression error at $t$-step will be much smaller than the constant error bound obtained by plain error feedback. In particular, \citet{richtarik2021ef21} showed that if the sequence of $\{\wb_t\}$ achieves a linear rate of convergence, the compression error of the Markov compression can converge to $0$ as $t\rightarrow \infty$.

\textbf{Worker-side model update:}
\citet{richtarik2021ef21} only considered a one-way compression, i.e., workers compress the gradient difference $\gb_t^{(i)}-\gbh_{t-1}^{(i)}$ and update $\cC(\gb_t^{(i)}-\gbh_{t-1}^{(i)})$ to the server, while the server only updates model $\xb_{t+1}$ and broadcasts it to workers. In this case, the server broadcast still results in high communication overhead. Therefore, to further reduce the communication cost, we change the model update style from server-side update to worker-side update, i.e., the model $\xb_t$ is stored and updated on each worker instead of on the server. 
The reason is that, in server-side update style, the server will need to compress the model update parameter $\Vbh_t^{-1/2}\mb_t$ for reducing the communication cost, however, it is difficult to obtain the convergence of the underlying sequence $\Vbh_t^{-1/2}\mb_t$, which is essential for the proof of the Markov compression sequence. Thus, it is difficult to get a better compression error bound for the model update. In our design, the server simply aggregates the compressed gradients and compressed it again before sending back to the workers. The compression errors via Markov compressor are directly related to fresh stochastic gradients or aggregated gradients, whose convergences are much easier to obtain. 
This two-way compression together with worker-side model update take the full advantage of the Markov compressor for establishing the theoretical guarantees of our proposed method.

\textbf{Algorithm overview:}
Combining Markov compression sequence and worker-side model update design, we propose a new communication-compressed distributive gradient method, which is summarized in Algorithm \ref{alg:compadam1}. Specifically, at $t$-th iteration, each worker first computes the stochastic gradient $\gb_t^{(i)}$  with a mini-batch size $\tau$, then builds the Markov compression sequence $\gbh_{t}^{(i)}$ with $\cb_t^{(i)} = \cC(\gb_t^{(i)}-\gbh_{t-1}^{(i)})$ and $\gbh_{t-1}^{(i)}$. Note that the Markov compression sequence $\gbh_{t}^{(i)}$ is no longer bits-compressed, therefore, only $\cb_t^{(i)}$ is being uploaded to the server. On the other hand, the server updates the aggregated compressed gradient by $\gbh_t = \gbh_{t-1} + \frac{1}{n}\sum_{i=1}^n \cb_t^{(i)}$ as the last iterate $\gbh_{t-1}$ has been stored locally in the server. The server then builds another Markov compression sequence $\gbt_{t}$ based on $\gbt_{t-1}$ and $\cb_t= \cC(\gbh_t - \gbt_{t-1})$ and sends $\cb_t$ to all workers. Similarly, each worker can recover the Markov compression sequence $\gbt_t$ upon receiving $\cb_t$, since the last iterate $\gbt_{t-1}$ has been locally stored in each worker. Then each worker uses the double-compressed gradient $\gbt_t$ for updating model $\xb_{t+1}$ as in standard AMSGrad.

Algorithm \ref{alg:compadam1} is indeed communication-efficient for distributed learning. It improves the communication efficiency by reducing the communication bits during transmission. Only the compressed vectors instead of the full precision vectors are transferred for both server-to-worker and worker-to-server communications. 
Specifically, if we adopt the scaled sign compressor, CD-Adam only takes 1-bit\footnote{Rigorously, the scaling number will also take 32-bits for communication, so the overall cost for compressing a $d$-dimensional vector should be $32 + d$ bits.} instead of 32-bits per dimension in each communication round. This greatly reduces the communication costs for distributed implementation of adaptive methods. Note that compared with 1-bit Adam \citep{tang20211}, which performs a few epochs of uncompressed Adam at the beginning of training, our proposed method is much more communication-efficient (see Figure \ref{fig:bit_example}) as we start the compression from the very first iteration.

\begin{algorithm}[h!]
  \caption{ 
  \textbf{C}ommunication-Compressed \ \qquad\qquad \\ \textbf{D}istributed \textbf{Ada}ptive Gradient \textbf{M}ethod (CD-Adam)}
  \label{alg:compadam1}
  \begin{flushleft}
        \textbf{Input:} initial point $\xb_1$, non-increasing step size $\{\alpha_t\}_{t=1}^T$,  $\beta_1,  \beta_2, \nu$, batch size $\tau$, compressor $\cC(\cdot)$.
        \end{flushleft}
  \begin{algorithmic}[1]
        \STATE $\gb_0^{(i)} \gets 0, \mb_0\gets 0$, $\vb_0\gets 0, \gbh_0^{(i)}=\cC(\gb_0^{(i)}), \gbt_0=\cC(\frac{1}{n} \sum_{i=1}^n \gbh_0^{(i)})$
      \FOR{$t=1$ to $T$}
        \STATE \textbf{(On $i$-th worker)}
        \STATE Compute local stochastic gradient \texttt{$\gb_t^{(i)}$} with batch size $\tau$
        \STATE Compress \texttt{ $\cb_t^{(i)}=\cC(\gb_{t}^{(i)}-\gbh_{t-1}^{(i)}) $}
        \STATE Send $\cb_t^{(i)}$ to the server and update local state \texttt{ $\gbh_t^{(i)}= \gbh_{t-1}^{(i)}+\cb_t^{(i)}$}
        \STATE \textbf{(On Server)}
        \STATE Update \texttt{ $\gbh_t= \gbh_{t-1}+ \frac{1}{n} \sum_{i=1}^{n} \cb_t^{(i)}$}
        \STATE Compress $\cb_t=\cC(\gbh_t-\gbt_{t-1})$
        \STATE Send $\cb_t$ to all the workers and update local state $\gbt_t: \gbt_t= \gbt_{t-1}+\cb_t$
        \STATE \textbf{(On $i$-th worker)}
        \STATE Update \texttt{$\gbt_t= \gbt_{t-1}+\cb_t$ }
        \STATE \texttt{$\mb_t=\beta_1\mb_{t-1}+(1-\beta_1)\gbt_t\;$}
        \STATE \texttt{$\vb_t=\beta_2\vb_{t-1}+(1-\beta_2)\gbt_t^2\;$}
        \STATE \texttt{$\hat{\vb}_t=\max(\hat{\vb}_{t-1},\vb_t)$}
  \STATE Update \texttt{$\xb_{t+1}=\xb_t-\alpha_t\Vbh_t^{-1/2}\mb_t $} with \texttt{$\hat{\Vb}_t=\diag(\hat{\vb}_t + \nu)$}
      \ENDFOR

  \end{algorithmic}
\end{algorithm}

\begin{figure}[t]
 \centering
  \includegraphics[width=0.5\textwidth]{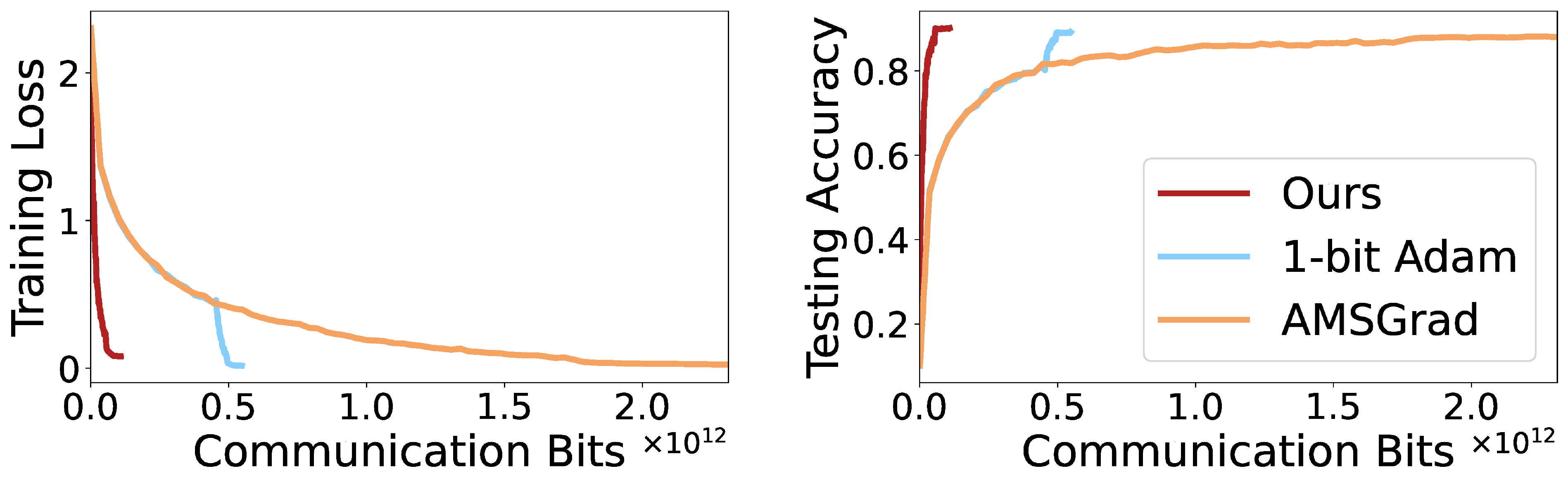}
   \setlength{\abovecaptionskip}{-5pt}
  \setlength{\belowcaptionskip}{-5pt}
 \caption{Comparison of training loss and testing accuracy against communication bits of training ResNet-18 on CIFAR-10. Our proposed method achieves around $32\times$ communication cost improvement over the original AMSGrad and around $5\times$ over 1-bit Adam.}
 \label{fig:bit_example}
\end{figure}

\section{Convergence Analysis}
In this section, we present the convergence results of our proposed Communication-Compressed Distributed Adaptive Gradient Method (CD-Adam). 
Before we jump into the main theorem, let us first introduce the additional assumptions needed.

\begin{assumption}[Smoothness]\label{as:smooth}
Each component loss function on the $i$-th worker $f_i(\xb)=\EE[f(\xb;\xi^{(i)})]$ is $L$-smooth, i.e., $ \forall \bm{x}, \bm{y} \in \RR^d$,
    \begin{align*}
    \big|f_i(\xb)-f_i(\yb)-\langle\nabla f_i(\yb), \xb-\yb\rangle\big| \leq \frac{L}{2}\|\xb-\yb\|_2^2.
    \end{align*}
\end{assumption}

Assumption~\ref{as:smooth} is a standard assumption for nonconvex stochastic optimization  \citep{koloskova2019decentralized,basu2019qsparse,richtarik2021ef21}.
Note that the $L$-smoothness assumption on each worker's loss $f_i(\cdot)$ implies the $L$-smoothness condition on $f(\cdot)$. Assumption~\ref{as:smooth} also implies the $L$-gradient Lipschitz condition, i.e., $\|\nabla f_i(\bm{x}) - \nabla f_i(\bm{y}) \|_2 \leq L \|\bm{x} - \bm{y} \|_2$.

\begin{assumption}[Bounded gradient]\label{as:bound-g}
Each component loss function on the $i$-th worker $f_i(\xb)$ 
has $G$-bounded stochastic gradient on $\ell_2$ and has $G_\infty$-bounded stochastic gradient on $\ell_\infty$, i.e., for all $\xi$,
\begin{align*}
\|\nabla f_i(\xb ; \xi)\|_2 \leq G, \qquad \|\nabla f_i(\xb ; \xi)\|_\infty \leq G_\infty.
\end{align*}
\end{assumption}
Assumption \ref{as:bound-g} is a common assumption in adaptive gradient literature \citet{kingma2014adam,chen2018on,1808.05671,1808.02941}. It also implies the bounded gradient for $f(\xb)$, i.e., $\|\nabla f(\xb ; \xi)\|_2 \leq G$, $\|\nabla f(\xb ; \xi)\|_\infty \leq G_\infty$. 

\begin{assumption}[Bounded local variance]\label{as:bound-variance}
Each stochastic gradient on the $i$-th worker
$\gb^i = \nabla f_i(\xb, \xi^i)$ has a bounded variance, i.e., $\forall \xb$,
\begin{align*}
\EE_{\xi^i \sim \cD_i} \big\|\nabla f_i (\xb,\xi^i)-\nabla f_i(\xb)\big\|^2\leq \sigma^2.
\end{align*}

\end{assumption}
Assumption~\ref{as:bound-variance} implies that the variance of stochastic gradients is bounded on each worker.  
The assumption of bounded variance for stochastic gradient has also been used in \citep{ basu2019qsparse,karimireddy2019error,horvath2019stochastic}.

\begin{figure*}[ht!]
 \centering
 \subfigure{\includegraphics[width=1.0\textwidth]{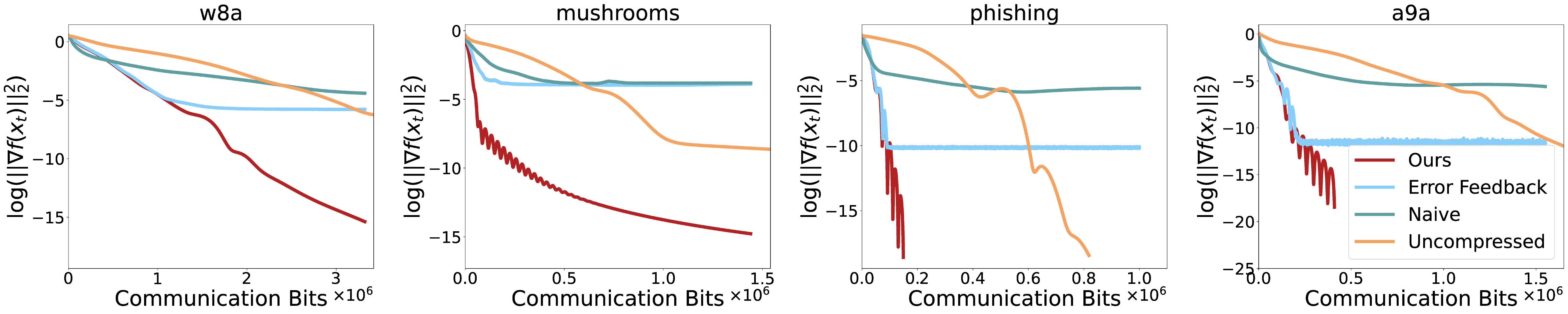}}
  \subfigure{\includegraphics[width=1.0\textwidth]{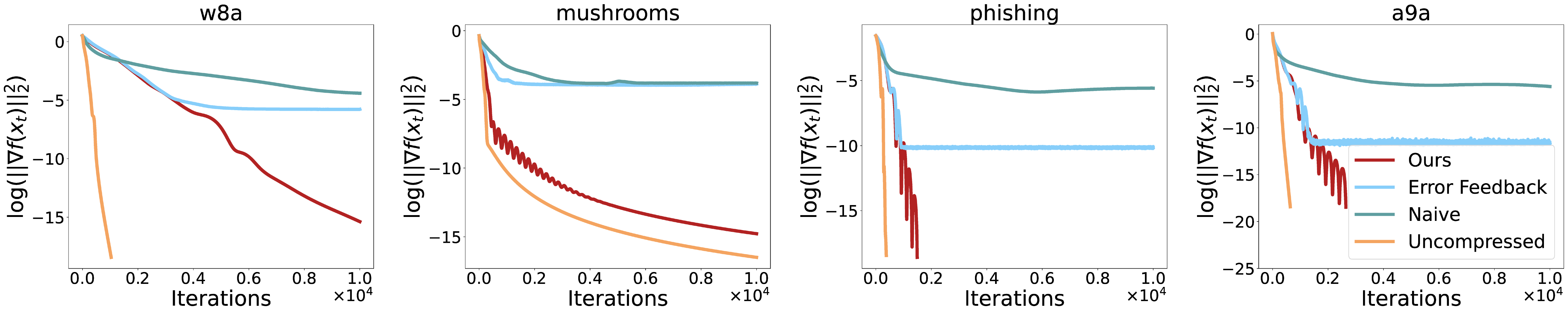}}
  \setlength{\abovecaptionskip}{-5pt}
  \setlength{\belowcaptionskip}{-5pt}
 \caption{Gradient norm comparison of different compressing strategies on nonconvex logistic regression trained by AMSGrad with scaled sign compressor. The upper row shows the norm convergence with respect to the communication cost, and the lower row is with respect to the training iteration.}
 \label{fig:lr_sign_bit}
 \smallskip
\vspace{2pt}
\end{figure*}

Now we are ready to present our main theorem.
\begin{theorem}\label{thm}
Under Assumptions~\ref{as:compressor}, \ref{as:smooth}, \ref{as:bound-g} and \ref{as:bound-variance}, suppose $\alpha_t = \alpha$, denote $N$ as the number of data samples on each worker, the iterates of Algorithm \ref{alg:compadam1} achieves an $\epsilon$-stationary point of \eqref{lossfunc}, i.e., $\min_{t \in [T]} \EE \big[\| \nabla f(\xb_{t}) \|_2^2 \big]\leq \epsilon$, if
\begin{align}\label{eq:thm}
    \alpha &\leq\frac{n\epsilon}{6nM_3+
6M_4\sigma^2}, \tau \geq \bigg\lceil \frac{N(3M_5 \sigma)^2}{(N-1)\epsilon^2+ (3M_5 \sigma)^2} \bigg\rceil, \notag\\
    T &\geq \bigg\lceil\frac{36M_1 M_3}{\epsilon^2}+ \frac{36M_1M_4\sigma^2}{n\epsilon^2}+\frac{3M_2}{\epsilon} \bigg\rceil,
\end{align}
where
\begin{align*}
M_1 & = C\cdot\Delta f,\quad
M_2 = \frac{C G \Tilde{G}}{(1-\beta_1)\sqrt{\nu}}, \\
M_3 & = \frac{32CC_1 \Tilde{G}^2}{\nu} + \frac{2\sqrt{\pi}C L G \Tilde{G} \sqrt{d}}{\nu(1-\sqrt{\pi})^2} \notag\\
    M_4 & = \frac{4CC_1}{\nu},
    M_5 = \frac{4\sqrt{\pi} C G}{\sqrt{\nu}(1-\sqrt{\pi})^2},
\end{align*}
and 
\begin{align*}
&\Delta f = f(\xb_1)- \inf_x f(\xb), 
\Tilde{G} = \frac{(1+\sqrt{\pi})^2}{(1-\sqrt{\pi})^2} G,\\
&C = 2 (\Tilde{G}_\infty^2+\nu)^{1/2}, C_1 = 2L+3L\bigg(\frac{\beta_1}{1-\beta_1}\bigg)^2,\\ 
&\Tilde{G}_\infty = \frac{(1+\sqrt{\pi})^2}{(1-\sqrt{\pi})^2} G_\infty.
\end{align*}
\end{theorem}
\smallskip

\begin{remark}
Theorem \ref{thm} suggests that Algorithm \ref{alg:compadam1} reaches the $\epsilon$-stationary point of \eqref{lossfunc} with  $O(1/\epsilon^2)$ iterations, which matches the iteration complexity of uncompressed vanilla AMSGrad. This suggests that our proposed method is indeed provably efficient, and the fully compressed distributed adaptive gradient method can reach $\epsilon$-stationary point with the same iteration complexity as the uncompressed counterpart without applying any variance-freezing tricks. 
\end{remark}

\begin{remark}
Theorem \ref{thm} also suggests that a larger number of workers $n$ can lead to a better iteration complexity as shown in \eqref{eq:thm}.
In addition, the condition on the mini-batch size $\tau$ is needed for controlling the stochastic variance in a mini-batch. Note that $\tau$ is always smaller than the number of data samples on each worker $N$ as the full batch gradient does not incur stochastic variance. 
\end{remark}

\vspace{-5pt}
\section{Experiments}
In this section, we present empirical results of our proposed communication-compressed distributed adaptive gradient method and compare it with other communication-compressed distributed baselines. Specifically, we first present the experimental results on a nonconvex logistic regression problem as an illustrative case study. Then we present the deep learning experiments of image classification on standard benchmarks. 

\begin{figure*}[ht]
 \centering
  \includegraphics[width=1.0\textwidth]{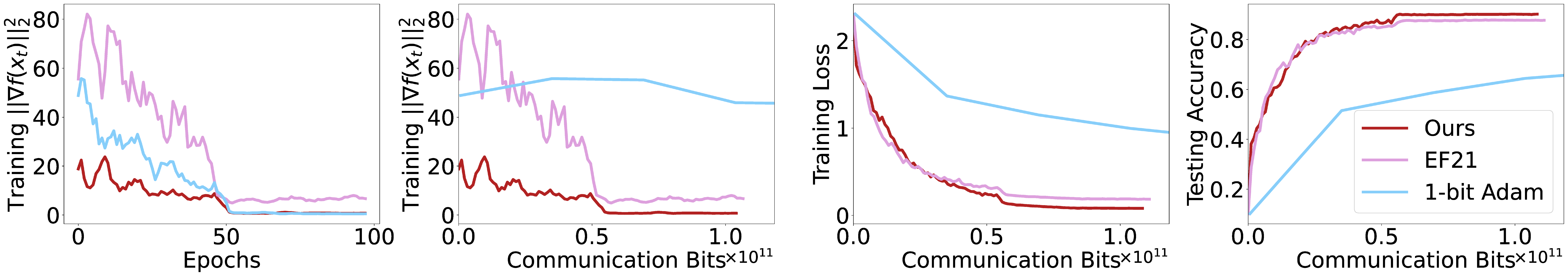}
    \setlength{\abovecaptionskip}{-5pt}
  \setlength{\belowcaptionskip}{-5pt}
 \caption{Comparison on gradient norm, training loss and testing accuracy among our proposed method and the baselines when training \texttt{ResNet-18} model on \texttt{CIFAR-10}.}
 \label{fig:resnet_showcase}
\end{figure*}

\subsection{Illustrative case study: nonconvex logistic regression}\label{subsec:ncvx-sign}
We consider minimizing the following logistic regression problem with a nonconvex regularizer similar to \citet{richtarik2021ef21},
\begin{align}\label{eq:logistic}
f(\xb) & =\frac{1}{n} \sum_{i=1}^n \log \Big(1+\exp \big(-y_i a_i^\top x\big)\Big)+\lambda \sum_{j=1}^d \frac{x_j^2}{1+x_j^2},
\end{align}
where $a_i \in \RR^d$, $y_i$ are training data with value $\pm 1$, and $\lambda$ is the regularizer parameter. We set $\lambda=0.1$ in the following nonconvex logistic regression experiments.

We use datasets \texttt{phishing}, \texttt{mushrooms}, \texttt{a9a} and \texttt{w8a} from LibSVM \citep{10.1145/1961189.1961199}. We equally separate each dataset to $n=20$ parts, and we use full batch gradients in this experiment. For each method, we choose the best step size starting from $0.001$ and increase it by adding $0.002$ till achieving $0.01$. 

Let's first study the effect of using different compression strategies on AMSGrad, though the other strategies do not enjoy any theoretical guarantees. Figure \ref{fig:lr_sign_bit} compares the gradient norm convergence of our proposed CD-Adam with AMSGrad using error feedback, naive compression and without compression at all. For all compressed methods, we use the scaled sign as a canonical example of biased compressor\footnote{Additional experiments based on the Top-$k$ biased compressor can be found in the supplemental materials.} $\cC(\cdot)$. We observe that our proposed CD-Adam achieves the best performances on all four datasets against the other three compression strategies. Specifically, if we compare the communication cost, CD-Adam achieves much better communication efficiency against uncompressed AMSGrad. Compared with the error feedback and the naive compression, CD-Adam still achieves a much smaller gradient norm under the same communication budgets. If we compare the performance of different compression strategies under the same iteration complexity, our proposed algorithm achieves roughly the same final gradient norm as the uncompressed AMSGrad across all four datasets. In fact, it achieves a much better convergence result comparing with the error feedback and the naive compression strategies, whose gradient norm stops decreasing at the early stages.  

\subsection{Deep learning experiments on image classification}\label{subsec:dl}
We compare CD-Adam with several state-of-the-art  communication-compressed distributed learning algorithms that are provably efficient, including: (1) EF21 \citep{richtarik2021ef21}\footnote{The comparison to SGD with Error Feedback has been presented in \citet{richtarik2021ef21}.} (2) 1-bit Adam \citep{tang20211}. We adopt CD-Adam with scaled sign compressor in this experiment. Note that the original EF21 paper \citep{richtarik2021ef21} adopts the top-K compressor as the base compressor $\cC(\cdot)$ with only worker-to-server compression applied. For a fair comparison, we further extend EF21 to allow both worker-to-server and server-to-worker compression and choose K$=0.016d$ such that the communication compression ratio keeps roughly the same as the scaled-sign based compressor. 

We train \texttt{CIFAR10} \citep{krizhevsky2009learning} using three popular models: \texttt{ResNet-18} \citep{he2016deep}, \texttt{VGG-16} \citep{simonyan2014very} and \texttt{WideResNet-16-4} \citep{zagoruyko2016wide}. The image classification dataset \texttt{CIFAR10} includes a training set of $50000$ images and a test set of $10000$ images, where each image is assigned one of the $10$ labels. The dataset is split into $n=8$ equal parts, which are distributed to $8$ workers. The mini-batch size for each worker is set to be 128. We set the learning rate as $1\times 10^{-1}$ for SGD, which is used in EF21, and $1\times 10^{-4}$ for 1-bit Adam and our proposed CD-Adam. We also set $\beta_1 = 0.9$ and $\beta_2 = 0.99$ for 1-bit Adam and our CD-Adam. For 1-bit Adam, its warm-up epochs are set as $13$ according to its original paper \citep{tang20211}. The normalized weight decay is set to $5\times 10^{-4}$ for all methods. We test all methods for $100$ epochs and decay the learning rate by $0.1$ at the 50th and 75th epoch. Due to the space limit, we leave the experimental results on \texttt{VGG-16} and \texttt{WideResNet-16-4} models in the supplemental materials.

We show that our proposed algorithm enjoys a fast convergence speed with high accuracy and less communication cost comparing with EF21 and 1-bit Adam. The first and second plots in Figure \ref{fig:resnet_showcase} demonstrate the training gradient norm in terms of epochs and communication bits respectively. We can observe that our proposed CD-Adam obtains a smaller gradient norm than EF21 and 1-bit Adam under the same epoch or communication budget. Note that since 1-bit Adam will need to run uncompressed Adam for a few epochs as a warm-up, its per communication bits performance is actually much worse (has been shown in Figure \ref{fig:bit_example}). The third plot in Figure \ref{fig:resnet_showcase} shows the training loss against communication bits on \texttt{CIFAR-10} dataset. At the early stage of the training process, CD-Adam and EF21 obtain a similar speed of reducing the training loss, while 1-bit Adam is much slower due to the warm-up process. At later stages, CD-Adam shows a clear advantage compared with EF21. The last plot in Figure \ref{fig:resnet_showcase} shows the test accuracy against communication bits. Similar phenomenon can be observed that  CD-Adam achieves the overall best test accuracy compared with EF21 and 1-bit Adam, under the same communication budget. These results suggest that our proposed CD-Adam is indeed much more communication-efficient while maintaining a high accuracy compared with other communication compression baselines.

\section{Conclusion}
In this paper, we propose a communication-compressed distributed adaptive gradient method which solves the bottleneck of applying communication compression strategies for adaptive gradient methods in distributed settings. We provide theoretical convergence analysis in the nonconvex stochastic optimization setting and show that our proposed algorithm converges to an $\epsilon$-stationary point with the same iteration complexity as the uncompressed vanilla AMSGrad. Furthermore, compared with prior work which adopts variance-freezed Adam, our proposed algorithm is fully adaptive during the entire training process. It does not require any warm-up procedure with uncompressed communication, leading to better empirical results in terms of both performance and communication cost.

\bibliographystyle{apalike}
\bibliography{reference.bib}

\appendix\onecolumn

\section{Compression Schemes} \label{appendix}
Here we provide several compressors satisfy Assumption~\ref{as:compressor}.

\textbf{Rand-$k$} \citep{efsparse}: For $1\leq k \leq d$, denotes
$S$ as chosen $k$ elements randomly in $[d]=\{1,2,...,d\}$, i.e. $S \subseteq [d]=\{1,2,...,d\}$, $|S|=k$.
Denotes $\eb_1, \eb_2,...,\eb_d$ as standard unit basis vectors in $\RR^d$, for $\xb \in \RR^d$, the compressor $\cC_{\text{rand}}:\RR^d \rightarrow \RR^d$ is defined as$$ \cC_{\text{rand}}(\xb)= \sum_{i \in S} \xb_i \eb_i.
$$
For Rand-$k$ compressor, we have 
\begin{align}\label{rand-k}
\EE\big[\big\|\cC_{\text{Rand}}(\xb)-\xb\big\|_2^2\big] = \EE\bigg[\bigg\|\sum_{i\notin S} \xb_i\eb_i\bigg\|_2^2\bigg] = \bigg(1-\frac{k}{d}\bigg)\|\xb\|_2^2.
\end{align}

\textbf{Top-$k$} \citep{efsparse}:
For $1\leq k \leq d$ and $\xb \in \RR^d$,
the order of $\xb$ is ordered by the magnitude of coordinates with $|x_{(1)}|\leq |x_{(2)}|\cdots \leq |x_{(d)}|$. The compressor $\cC_{\text{Top}}:\RR^d \rightarrow \RR^d$ is defined as $$\cC_{\text{Top}}(\xb)= \sum_{i=d-k+1}^d \xb_{(i)} \eb_{(i)}.
$$
For Top-$k$ compressor, considering the defition of rand-$k$ compressor, we have 
\begin{align*}
\big\|\cC_{\text{Top}}(\xb)-\xb\big\|_2 \leq \big\|\cC_{\text{Rand}}(\xb)-\xb\big\|_2,
\end{align*}
thus we conclude $$ \big\|\cC_{\text{Top}}(\xb)-\xb\big\|_2 \leq \bigg(1-\frac{k}{d}\bigg)\|\xb\|_2^2.
$$

\textbf{Scaled sign} \citep{karimireddy2019error}:
For $1\leq k \leq d$ and $\xb \in \RR^d$, the scaled sign compressor $\cC_{\sign}:\RR^d \rightarrow \RR^d$ is defined as$$
\cC_{\sign}(\xb)=\|\xb\|_1\cdot \sign(\xb) /d,
$$
where the $\sign(\xb)$ takes sign of each coordinates. For scaled sign compressor, we have
\begin{align}
\big\|\cC_{\sign}(\xb)-\xb\big\|_2^2 & = \sum_{i=1}^d \bigg(\frac{1}{d}\sum_{j=1}^d |x_j|-|x_i| \bigg)^2 \notag\\
& = \sum_{i=1}^d \bigg(\frac{1}{d}\sum_{j=1}^d |x_j|\bigg)^2 - 2 \bigg(\sum_{i=1}^d |x_i|\bigg) \bigg(\frac{1}{d}\sum_{j=1}^d |x_j|\bigg) + \sum_{i=1}^d |x_i|^2 \notag\\
& = d \cdot \frac{1}{d} \|\xb\|_2^2 - \frac{2}{d} \|\xb\|_1^2 +\|\xb\|_2^2 \notag\\
& = \|\xb\|_2^2-\frac{1}{d}\|\xb\|_1^2 \notag\\
& = \bigg( 1-\frac{\|\xb\|_1^2}{d\|\xb\|_2^2}\bigg) \|\xb\|_2^2.
\end{align}
Thus we conclude that $$\big\|\cC_{\sign}(\xb)-\xb\big\|_2^2 = \bigg( 1-\frac{\|\xb\|_1^2}{d\|\xb\|_2^2}\bigg) \|\xb\|_2^2.
$$

\section{Main Theorem for Algorithm~\ref{alg:compadam1}} \label{supp:thm}
We provide more following lemmas to support the proof of the main theorem. 

\begin{lemma}
\label{lm: compressor}
For compressors $\cC(\cdot)$ satisfying Assumption \ref{as:compressor}, we have $\EE \big[\big\| \cC(\xb)\big\|_2\big] \leq (1+\sqrt{\pi}) \|\xb \|_2$.
\end{lemma}
Lemma~\ref{lm: compressor} can be easily verified by applying triangle inequality to Assumption \ref{as:compressor}. 

\begin{lemma}\label{lm:mVbound-main}
Under Assumptions \ref{as:compressor} and \ref{as:bound-g}, we have $\|\nabla f(\xb)\|_2 \leq G$, $\|\gbh_{t}\|_2 \leq \frac{1+\sqrt{\pi}}{1-\sqrt{\pi}} G=\hat{G}$,  $\|\gbt_{t}\|_2 \leq \frac{1+\sqrt{\pi}}{1-\sqrt{\pi}} \hat{G}=\Tilde{G}$,
$\|{\vb}_t\|_2 \leq \Tilde{G}^2$ and $\|\mb_{t}\|_2 \leq \Tilde{G}$. Additionally, we have $\|\nabla f(\xb)\|_\infty \leq G_\infty$, $\|\gbh_{t}\|_\infty \leq \frac{1+\sqrt{\pi}}{1-\sqrt{\pi}} G_\infty=\hat{G}_\infty$,  $\|\gbt_{t}\|_\infty \leq \frac{1+\sqrt{\pi}}{1-\sqrt{\pi}} \hat{G}_\infty=\Tilde{G}_\infty$,
$\|{\vb}_t\|_\infty \leq \Tilde{G}_\infty^2$ and $\|\mb_{t}\|_\infty \leq \Tilde{G}_\infty$.
\end{lemma}
Lemma \ref{lm:mVbound-main} shows that the compressed gradients and the update parameters can be bounded by constants, it is a quite standard result for the adam-type convergence analysis.

\begin{lemma}\label{lm:mini-batch-var}
For the mini-batch stochastic gradient on each worker, its variance satisfies 
\begin{align*}
    \EE [\|\gb_t^{(i)} - \nabla f_i(\xb_t)\|_2^2] \leq \sigma^2 \frac{N-\tau}{\tau(N-1)},
\end{align*}
where $\tau$ denotes the batch size and $N$ denotes the total amount of data on each worker.
\end{lemma}
Lemma \ref{lm:mini-batch-var} shows the bounded variance of local stochastic gradient with batch size $\tau$.
\begin{lemma}\label{lm:g-gh-gt}

Let $\beta_1, \beta_2$ be the weight parameters such that
\begin{align*} 
&\mb_t=\beta_1\mb_{t-1}+(1-\beta_1)\gbt_t,\\
&\vb_t=\beta_2\vb_{t-1}+(1-\beta_2)\gbt_t^2,
\end{align*}
set $\alpha_t=\alpha$, $t=1,\ldots,T$ be the step sizes. Under Assumption~ \ref{as:bound-g} and \ref{as:bound-variance}, we have the following two results:
\begin{align*}
    \sum_{t=1}^T \alpha_t^2 \big\|\Hat{\Vb}_t^{-1/2}\mb_t\big\|_2^2
    & \leq \frac{32T\alpha^2 \Tilde{G}^2}{\nu} + \frac{4T\alpha^2\sigma^2}{\nu n} + \frac{4\alpha^2}{\nu}\sum_{t=1}^T\EE[\|\nabla f(\xb_t)\|_2^2],
 \end{align*}
and
 \begin{align*}
     \sum_{t=1}^T \alpha_t^2 \big\|\Hat{\Vb}_t^{-1/2}\gbt_t\big\|_2^2 
     & \leq \frac{32T\alpha^2 \Tilde{G}^2}{\nu} + \frac{4T\alpha^2\sigma^2}{\nu n} + \frac{4\alpha^2}{\nu}\sum_{t=1}^T\EE[\|\nabla f(\xb_t)\|_2^2],
 \end{align*}
\end{lemma}
Lemma \ref{lm:g-gh-gt} provides upper bounds for the summation of $\alpha_t^2 \big\|\Vbh_t^{-1/2} \mb_t\big\|_2$ and $\alpha_t^2 \big\|\Vbh_t^{-1/2} \gbt_t\big\|_2$.

\begin{lemma}\label{lm:gb-gbh}
Denote $\gb_t= \frac{1}{n} \sum_{i=1}^n \gb_t^{(i)}$ as the aggregated fresh gradient and denote $\gbh_t =\frac{1}{n} \sum_{i=1}^n \gbh_t^{(i)}$ as the aggregated worker-side compressed gradient.
The compression error of the worker-side compression $ \|\gbh_t-\gb_t\|_2$ satisfies
\begin{align}
    \EE [\|\gbh_t-\gb_t\|_2] \leq  \frac{L \Tilde{G} \sqrt{\pi}}{1-\sqrt{\pi}} \cdot \alpha_1 \sqrt{d}\nu^{-1/2}+ \frac{2\sqrt{\pi}\sigma}{1-\sqrt{\pi}}\sqrt{\frac{N-\tau}{\tau(N-1)}},
\end{align}
where $N$ is the total number of data samples on each worker and $\tau$ is the batch-size.
\end{lemma}

\begin{lemma}\label{lm:gbh-gbt}
The compression error of the server-side compression  $\|\gbt_t-\gbh_t\|_2$ satisfies
\begin{align}
    \EE [\|\gbt_t-\gbh_t\|_2] \leq \frac{ L \Tilde{G}\sqrt{\pi} (1+\sqrt{\pi})}{(1-\sqrt{\pi})^2} \cdot\alpha_1 \sqrt{d} \nu^{-1/2}+ \frac{2\sqrt{\pi}(1+\sqrt{\pi})\sigma}{(1-\sqrt{\pi})^2} \sqrt{\frac{N-\tau}{\tau(N-1)}},
\end{align}
where $N$ is the total number of data samples on each worker and $\tau$ is the batch-size.
\end{lemma}

Lemma \ref{lm:gb-gbh} and \ref{lm:gbh-gbt} provide upper bounds for the worker-side and the server-side compression errors respectively.

Following \cite{1808.05671}, we define an auxiliary sequence $\zb_t$ as follows: let $\xb_0 = \xb_1$, and for each $t \geq 1$, 
\begin{align}
    \zb_{t} = \xb_{t} + \frac{\beta_1}{1-\beta_1}(\xb_{t} - \xb_{t-1}). \label{def:z}
\end{align}
The following three lemmas directly follows Lemma 7.1, 7.2, and 7.3 from \citet{1808.05671}. 
\begin{lemma}\label{lm3-B}
Let $\zb_t$ be defined in \eqref{def:z}. For $t \geq 2$, we have
\begin{align}
    \zb_{t+1}-\zb_{t}&= \frac{\beta_1}{1 - \beta_1}\Big[\Ib - \big(\alpha_t\Vbh_{t}^{-1/2}\big) \big(\alpha_{t-1}\Vbh_{t-1}^{-1/2}\big)^{-1}\Big](\xb_{t-1} - \xb_t) - \alpha_t\Vbh_{t}^{-1/2} \gbt_t, \label{lm:7.3.1}
\end{align}
and 
\begin{align}
   \zb_{t+1}-\zb_{t} &  = 
    \frac{\beta_1}{1 - \beta_1}\big(\alpha_{t-1}\Vbh_{t-1}^{-1/2} - \alpha_t\Vbh_{t}^{-1/2}\big)\mb_{t-1} - \alpha_t \Vbh_{t}^{-1/2} \gbt_t,\label{lm:7.3.2}
\end{align}
Specifically, for $t = 1$, we have
\begin{align}
    \zb_2 - \zb_1 = -\alpha_1 \Vbh_{1}^{-1/2}\gbt_1.
\end{align}
\end{lemma}

\begin{lemma}\label{lm4-B}
Let $\zb_t$ be defined in \eqref{def:z}. For $t \geq 2$, we have
\begin{align*}
\|\zb_{t+1}-\zb_{t}\|_2&\le\big\|\alpha_t\Vbh_t^{-1/2}\gbt_t\big\|_2+\frac{\beta_1}{1 - \beta_1}\|\xb_{t-1}-\xb_{t}\|_2.
\end{align*}
\end{lemma}

\begin{lemma}\label{lm5-B}
Let $\zb_t$ be defined in \eqref{def:z}. For $t \geq 2$, we have
\begin{align*}
\big\|\nabla f(\zb_t)-\nabla f(\xb_t)\big\|_2 \leq L \bigg(
\frac{\beta_1}{1 - \beta_1}\bigg)\cdot \|\xb_{t-1}-\xb_{t}\|_2.
\end{align*}
\end{lemma}

Together with previous Lemmas, now we are ready to prove Theorem \ref{thm}.

\begin{proof}[Proof of Theorem \ref{thm}]
Since $f$ is $L$-smooth, we have:
\begin{align} 
    f(\zb_{t+1}) & \leq  f(\zb_t)+\nabla f(\zb_t)^\top(\zb_{t+1}-\zb_t)+\frac{L}{2}\|\zb_{t+1}-\zb_t\|_2^2\notag \\
    & = f(\zb_t) +\nabla f(\xb_t)^\top(\zb_{t+1}-\zb_t)+(\nabla f(\zb_t) -\nabla f(\xb_t))^\top(\zb_{t+1}-\zb_t) +\frac{L}{2}\|\zb_{t+1}-\zb_t\|_2^2. \label{eq:fzt} 
\end{align}

When $t = 1$, for \eqref{eq:fzt} we have 
\begin{align} \label{eq:fz1}
    f(\zb_2) &\leq f(\zb_1) +\nabla f(\xb_1)^\top(\zb_2-\zb_1)+(\nabla f(\zb_1) -\nabla f(\xb_1))^\top(\zb_2-\zb_1)+\frac{L}{2}\|\zb_2-\zb_1\|_2^2 \notag\\
    & = f(\zb_1)-\nabla f(\xb_1)^\top\alpha_{1} \Vbh_{1}^{-1/2} \gbt_1 +2 L \big\|\alpha_1\Vbh_{1}^{-1/2} \gbt_1 \big\|_2^2.
\end{align}
When $t\geq 2$, we have
\begin{align} \label{eq:nablafzt}
\nabla f(\xb_t)^\top(\zb_{t+1}-\zb_t)     & =
    \nabla f(\xb_t)^\top\bigg[\frac{\beta_1}{1 - \beta_1}\big(\alpha_{t-1}\Vbh_{t-1}^{-1/2} - \alpha_t\Vbh_{t}^{-1/2}\big)\mb_{t-1} - \alpha_t \Vbh_{t}^{-1/2} \gbt_t\bigg]\notag \\
    & = \frac{\beta_1}{1-\beta_1}\nabla f(\xb_t)^\top\big(\alpha_{t-1}\Vbh_{t-1}^{-1/2} - \alpha_t\Vbh_{t}^{-1/2}\big)\mb_{t-1} - \nabla f(\xb_t)^\top\alpha_t \Vbh_{t}^{-1/2} \gbt_t \notag\\
    & = \frac{\beta_1}{1-\beta_1}\nabla f(\xb_t)^\top\big(\alpha_{t-1}\Vbh_{t-1}^{-1/2} - \alpha_t\Vbh_{t}^{-1/2}\big)\mb_{t-1}  + \nabla f(\xb_t)^\top\big(\alpha_{t-1} \Vbh_{t-1}^{-1/2}-\alpha_t \Vbh_{t}^{-1/2} \big) \gbt_t \notag\\
    & \qquad - \nabla f(\xb_t)^\top\alpha_{t-1} \Vbh_{t-1}^{-1/2} \gbt_t \notag\\
    &\leq \frac{1}{1-\beta_1}\|\nabla f(\xb_t)\|_2\cdot \big\|\alpha_{t-1}\Vbh_{t-1}^{-1/2} - \alpha_t\Vbh_{t}^{-1/2}\big\|_{\infty,\infty}\cdot \|\mb_{t-1}\|_2- \nabla f(\xb_t)^\top\alpha_{t-1} \Vbh_{t-1}^{-1/2} \gbt_t \notag \\
    & \leq 
    \frac{1}{1-\beta_1}   G \Tilde{G} \Big[\big\|\alpha_{t-1}\Vbh_{t-1}^{-1/2}\big\|_{\infty,\infty} - \big\|\alpha_t\Vbh_{t}^{-1/2}\big\|_{\infty,\infty}\Big] - \nabla f(\xb_t)^\top\alpha_{t-1} \Vbh_{t-1}^{-1/2} \gbt_t,
\end{align}
where the first equality holds due to \eqref{lm:7.3.2} in Lemma \ref{lm3-B}. The first inequality holds because for vector $\xb$, $\yb$ and positive diagonal matrix $\Ab$, we have $\xb^\top\Ab\yb\leq \|\xb\|_2\cdot\|\Ab\|_{\infty,\infty}\cdot\|\yb\|_2$. The second inequality holds due to Lemma~\ref{lm3-B} and $\alpha_{t-1} \Vbh_{t-1}^{-1/2} \succeq \alpha_{t} \Vbh_{t}^{-1/2}\succeq 0$. Thus for diagonal matrices  $\alpha_{t-1} \Vbh_{t-1}^{-1/2}$ and  $\alpha_{t} \Vbh_{t}^{-1/2}$, we directly conclude $\big\|\alpha_{t-1} \Vbh_{t-1}^{-1/2} - \alpha_{t} \Vbh_{t}^{-1/2}\big\|_{\infty,\infty} = \big\|\alpha_{t-1} \Vbh_{t-1}^{-1/2} \big\|_{\infty,\infty}- \big\| \alpha_{t} \Vbh_{t}^{-1/2}\big\|_{\infty,\infty}$.

Next we bound the last two terms in \eqref{eq:fzt}. 

\begin{align}
    &\qquad \big(\nabla f(\zb_t) -\nabla f(\xb_t)\big)^\top(\zb_{t+1}-\zb_t)+\frac{L}{2}\|\zb_{t+1} - \zb_t\|_2^2  \notag \\
    & \leq 
    \big\|\nabla f(\zb_t) -\nabla f(\xb_t)\big\|_2\cdot\|\zb_{t+1}-\zb_t\|_2 +\frac{L}{2}\|\zb_{t+1} - \zb_t\|_2^2 \notag \\
    & \leq
    L \bigg(\frac{\beta_1}{1 - \beta_1}\bigg)\|\xb_t-\xb_{t-1}\|_2 \cdot
    \bigg[\big\|\alpha_t\hat{\Vb}_t^{-1/2}\gbt_t\big\|_2+ \frac{\beta_1}{1 - \beta_1}\|\xb_t-\xb_{t-1}\|_2 \bigg] \notag\\ & \qquad +\frac{L}{2}\bigg[\big\|\alpha_t\Vbh_t^{-1/2}\gbt_t\big\|_2+\frac{\beta_1}{1 - \beta_1}\|\xb_{t-1}-\xb_{t}\|_2\bigg]^2 \notag\\
    & \leq 
    \frac{L}{2}\big\|\alpha_t\hat{\Vb}_t^{-1/2}\gbt_t\big\|_2^2 + \frac{3L}{2} \bigg(\frac{\beta_1}{1 - \beta_1}\bigg)^2\|\xb_t-\xb_{t-1}\|_2^2+L\big\|\alpha_t\Vbh_t^{-1/2}\gbt_t\big\|_2^2+L\bigg(\frac{\beta_1}{1 - \beta_1}\bigg)^2\|\xb_{t-1}-\xb_{t}\|_2^2 \notag\\
    & \leq 
    2L\big\|\alpha_t\hat{\Vb}_t^{-1/2}\gbt_t\big\|_2^2 + 3L \bigg(\frac{\beta_1}{1 - \beta_1}\bigg)^2\|\xb_t-\xb_{t-1}\|_2^2,\label{eq:fztlast2}
\end{align}
where the second inequality holds because of Lemma \ref{lm4-B} and \ref{lm5-B}, the third inequality holds by applying Young's inequality to the first term and applying Cauchy inequality to the second term.

For $t = 1$, recalling \eqref{eq:fz1}, taking expectation and rearranging terms, we have
\begin{align}
    &\EE[f(\zb_2) - f(\zb_1)] \notag \\
    &\leq \EE\bigg[-\nabla f(\xb_1)^\top\alpha_{1} \Vbh_{1}^{-1/2} \gbt_1
    +2L\big\|\alpha_1\Vbh_1^{-1/2}\gbt_1\big\|_2^2 \bigg]
    \notag \\
    & =
    \EE\bigg[-\nabla f(\xb_1)^\top\alpha_{1} \Vbh_{1}^{-1/2} \gb_1
    -\nabla f(\xb_1)^\top\alpha_{1} \Vbh_{1}^{-1/2}  (\gbh_1- \gb_1) -\nabla f(\xb_1)^\top\alpha_{1} \Vbh_{1}^{-1/2} (\gbt_1- \gbh_1) +2L\big\|\alpha_1\Vbh_1^{-1/2}\gbt_1\big\|_2^2 \bigg] \notag\\
    & \leq
    \EE\bigg[- \alpha_1 \|\nabla f(\xb_1)\|_2^2 \cdot (\Tilde{G}_\infty^2+\nu)^{-1/2}
    -\nabla f(\xb_1)^\top\alpha_{1} \Vbh_{1}^{-1/2} (\gbh_1- \gb_1)-\nabla f(\xb_1)^\top\alpha_{1} \Vbh_{1}^{-1/2} (\gbt_1- \gbh_1)  \notag\\
    & \qquad  +2L\big\|\alpha_1\Vbh_1^{-1/2}\gbt_1\big\|_2^2 \bigg],
    \label{thm10}
\end{align}
where the first term of the last inequality holds because 
\begin{align}
    \EE \big[-\alpha_1 \nabla f(\xb_1)^\top\Vbh_{1}^{-1/2} \gb_1 \big]  
    & = \EE \big[ -\alpha_1 \nabla f(\xb_1)^\top\Vbh_{1}^{-1/2} \nabla f(\xb_1) \big]\notag\\
    & \leq - \alpha_1 (\Tilde{G}_\infty^2+\nu)^{-1/2} \EE[\|\nabla f(\xb_1)\|_2^2]. \notag 
\end{align}

For $t\geq 2$, substituting \eqref{eq:nablafzt} and \eqref{eq:fztlast2} into \eqref{eq:fzt}, taking expectation and rearranging terms, we have
\begin{align}
    &\EE\bigg[f(\zb_{t+1})+\frac{G\Tilde{G} \big\|\alpha_{t} \Vbh_{t}^{-1/2}\big\|_{\infty,\infty}}{1-\beta_1} - \bigg(f(\zb_t) + \frac{G \Tilde{G} \big\|\alpha_{t-1} \Vbh_{t-1}^{-1/2}\big\|_{\infty,\infty}}{1-\beta_1}\bigg)\bigg] \notag \\
    &\leq \EE\bigg[-\nabla f(\xb_t)^\top\alpha_{t-1} \Vbh_{t-1}^{-1/2} \gbt_t +  2L\big\|\alpha_t\Vbh_t^{-1/2}\gbt_t\big\|_2^2 + 3L \bigg(\frac{\beta_1}{1 - \beta_1}\bigg)^2\|\xb_t-\xb_{t-1}\|_2^2\bigg]\notag \\
    & = \EE\bigg[-\nabla f(\xb_t)^\top\alpha_{t-1} \Vbh_{t-1}^{-1/2} \gb_t- \nabla f(\xb_t)^\top\alpha_{t-1} \Vbh_{t-1}^{-1/2}(\gbh_t-\gb_t) - \nabla f(\xb_t)^\top\alpha_{t-1} \Vbh_{t-1}^{-1/2}(\gbt_t-\gbh_t)\notag\\
    & \quad + 2L\big\|\alpha_t\Vbh_t^{-1/2}\gbt_t\big\|_2^2 + 3L \bigg(\frac{\beta_1}{1 - \beta_1}\bigg)^2\big\|\alpha_{t-1}\Vbh_{t-1}^{-1/2}\mb_{t-1}\big\|_2^2 \bigg]\notag\\
    & \leq \EE\bigg[-\alpha_{t-1}\big\|\nabla f(\xb_t)\big\|_2^2 (\Tilde{G}_\infty^2+\nu)^{-1/2}+2L\big\|\alpha_t\Vbh_t^{-1/2}\gbt_t\big\|_2^2 + 3L \bigg(\frac{\beta_1}{1 - \beta_1}\bigg)^2\big\|\alpha_{t-1}\Vbh_{t-1}^{-1/2}\mb_{t-1}\big\|_2^2\notag\\
     & \quad+ \alpha_{t-1} \nu^{-1/2} G \|\gbh_t-\gb_t\|_2 +\alpha_{t-1} \nu^{-1/2} G \|\gbt_t-\gbh_t\|_2
    \bigg],\label{thm11}
\end{align}
where the equality holds because $\gbt_t=
\gb_t+\big(\gbh_t-\gb_t\big)+\big(\gbt_t-\gbh_t\big) $. The second inequality holds because of 
\begin{align}
    \EE \big[ -\nabla f(\xb_t)^\top\alpha_{t-1} \Vbh_{t-1}^{-1/2} \gb_t \big] & =\EE\big[-\nabla f(\xb_t)^\top\alpha_{t-1} \Vbh_{t-1}^{-1/2} \EE[\gb_t] \big] \notag\\
    & = \EE \big[-\nabla f(\xb_t)^\top\alpha_{t-1} \Vbh_{t-1}^{-1/2} \nabla f(\xb_t) \big] \notag\\
    & \leq -\alpha_{t-1}  (\Tilde{G}_\infty^2+\nu)^{-1/2} \EE[\big\|\nabla f(\xb_t)\big\|_2^2],\\
    \EE \big[- \nabla f(\xb_t)^\top\alpha_{t-1} \Vbh_{t-1}^{-1/2}(\gbh_t-\gb_t) \big]
    & \leq \alpha_{t-1} \nu^{-1/2} \big\|\nabla f(\xb_t)\big\|_2 \cdot \EE [\|\gbh_t-\gb_t\|_2] \notag\\
    & \leq \alpha_{t-1} \nu^{-1/2} G \cdot \EE [\|\gbh_t-\gb_t\|_2], \\
    \EE \big[- \nabla f(\xb_t)^\top\alpha_{t-1} \Vbh_{t-1}^{-1/2}(\gbt_t-\gbh_t) \big]
    & \leq \alpha_{t-1} \nu^{-1/2} \big\|\nabla f(\xb_t)\big\|_2 \cdot \EE \|\gbt_t-\gbh_t\|_2 \notag\\
    & \leq \alpha_{t-1} \nu^{-1/2} G \cdot \EE [\|\gbt_t-\gbh_t\|_2].
\end{align}
Summing up \eqref{thm11} for $t=2$ to $T$, rearranging and adding with \eqref{thm10}, we have
\begin{align}
    & \quad (\Tilde{G}_\infty^2+\nu)^{-1/2}
    \sum_{t=1}^T \alpha_t \EE\big[\big\|\nabla f(\xb_t)\big\|_2^2\big] \notag \\
    &\leq 
    \EE\bigg[f(\zb_1) + \frac{G \Tilde{G} \big\|\alpha_1 \Vbh_{1}^{-1/2}\big\|_{\infty,\infty}}{1-\beta_1}
    - \bigg(f(\zb_{T+1})+\frac{G \Tilde{G} \big\|\alpha_T \Vbh_{T}^{-1/2}\big\|_{\infty,\infty}}{1-\beta_1}\bigg)\bigg]\notag \\
    & \qquad + 2L\sum_{t=1}^T \EE\big[\big\|\alpha_t\Vbh_t^{-1/2}\gbt_t\big\|_2^2\big] + 3L \bigg(\frac{\beta_1}{1 - \beta_1}\bigg)^2\sum_{t=1}^T \EE\big[\big\|\alpha_{t-1}\Vbh_{t-1}^{-1/2}\mb_{t-1}\big\|_2^2\big] \notag\\
    & \qquad+ \alpha_{t-1} \nu^{-1/2} G\cdot \sum_{t=1}^T \EE [\|\gbh_t-\gb_t\|_2] +\alpha_{t-1} \nu^{-1/2} G \cdot \sum_{t=1}^T \EE [\|\gbt_t-\gbh_t\|_2] \notag\\
    & \leq \EE\bigg[\Delta f + \frac{\alpha_1 G \Tilde{G} \nu^{-1/2}}{1-\beta_1}\bigg] +2L\sum_{t=1}^T \EE\big[\big\|\alpha_t\Vbh_t^{-1/2}\gbt_t\big\|_2^2\big] + 3L \bigg(\frac{\beta_1}{1 - \beta_1}\bigg)^2\sum_{t=1}^T \EE\big[\big\|\alpha_{t-1}\Vbh_{t-1}^{-1/2}\mb_{t-1}\big\|_2^2\big] \notag\\
    & \qquad+ \alpha_{t-1} \nu^{-1/2} G \cdot \sum_{t=1}^T \EE [\|\gbh_t-\gb_t\|_2] +\alpha_{t-1} \nu^{-1/2} G \cdot \sum_{t=1}^T \EE [\|\gbt_t-\gbh_t\|_2], \label{nabla-gbh-gbt}
\end{align}
where $\Delta f = f(\xb_1)- \inf_x f(\xb)$.

By Lemma \ref{lm:g-gh-gt}, with the constant stepsize $\alpha_t = \alpha$, we have
\begin{align} \label{Vmend}
    & \sum_{t=1}^T \alpha_t^2 \EE\big[ \big\|\Hat{\Vb}_t^{-1/2}\mb_t\big\|_2^2\big] \leq \frac{32T\alpha^2 \Tilde{G}^2}{\nu} + \frac{4T\alpha^2 \sigma^2}{\nu n} + \frac{4\alpha^2}{\nu}\sum_{t=1}^T\EE[\|\nabla f(\xb_t)\|_2^2],
\end{align}
and
\begin{align} \label{Vgend}
    & \sum_{t=1}^T \alpha_t^2 \EE\big[ \big\|\Hat{\Vb}_t^{-1/2}\gbt_t\big\|_2^2\big] \leq \frac{32T\alpha^2 \Tilde{G}^2}{\nu} + \frac{4T\alpha^2\sigma^2}{\nu n} + \frac{4\alpha^2}{\nu}\sum_{t=1}^T\EE[\|\nabla f(\xb_t)\|_2^2].
\end{align}

By Lemma~\ref{lm:gb-gbh} and \ref{lm:gbh-gbt}, we conclude the last two terms of \eqref{nabla-gbh-gbt} as
\begin{align} \label{gb-gbh-gbt}
    & \alpha_{t-1} \nu^{-1/2} G \cdot \sum_{t=1}^T  \EE[ \|\gbh_t-\gb_t\|_2] +
    \alpha_{t-1} \nu^{-1/2} G \cdot \sum_{t=1}^T \EE[\|\gbt_t-\gbh_t\|_2] \notag\\
    & \leq \frac{\sqrt{\pi}}{1-\sqrt{\pi}} L T G \Tilde{G} \alpha_1\alpha_{t-1} \sqrt{d} \nu^{-1} + \frac{\sqrt{\pi}(1+\sqrt{\pi})}{(1-\sqrt{\pi})^2}  L T G\Tilde{G} \alpha_1\alpha_{t-1} \sqrt{d} \nu^{-1} \notag\\
    & \qquad + \frac{2\sigma\sqrt{\pi}}{1-\sqrt{\pi}} T G \alpha_{t-1} \nu^{-1/2} \sqrt{\frac{N-\tau}{\tau (N-1)}}  + \frac{2\sigma\sqrt{\pi}(1+\sqrt{\pi})}{(1-\sqrt{\pi})^2} T G\alpha_{t-1} \nu^{-1/2} \sqrt{\frac{N-\tau}{\tau (N-1)}} \notag\\
    & \leq \frac{2\sqrt{\pi}}{\nu (1-\sqrt{\pi})^2} L G\Tilde{G} T \alpha_1\alpha_{t-1} \sqrt{d} + \frac{4\sqrt{\pi}}{\sqrt{\nu}(1-\sqrt{\pi})^2}  G T \alpha_{t-1} \cdot \sigma \sqrt{\frac{N-\tau}{\tau (N-1)}} .
\end{align}
Substituting \eqref{Vmend}, \eqref{Vgend} and \eqref{gb-gbh-gbt} into \eqref{nabla-gbh-gbt}, set the constant stepsize $\alpha_t=\alpha$ and rearranging, we have

\begin{align}
    & \quad (\Tilde{G}_\infty^2+\nu)^{-1/2}
    \sum_{t=1}^T \alpha \EE\big[\big\|\nabla f(\xb_t)\big\|_2^2\big] \notag \\
    & \leq \EE\bigg[\Delta f + \frac{\alpha G \Tilde{G} \nu^{-1/2}}{1-\beta_1}\bigg] + \bigg[2L+3L\bigg(\frac{\beta_1}{1-\beta_1}\bigg)^2\bigg] \frac{32T\alpha^2 \Tilde{G}^2}{\nu} \notag\\
    & \qquad +\bigg[2L+3L\bigg(\frac{\beta_1}{1-\beta_1}\bigg)^2\bigg] \frac{4T\alpha^2\sigma^2}{\nu n} +  \bigg[2L+3L\bigg(\frac{\beta_1}{1-\beta_1}\bigg)^2\bigg] \frac{4\alpha^2}{\nu}\sum_{t=1}^T\EE[\|\nabla f(\xb_t)\|_2^2] \notag\\
    & \qquad + \frac{2\sqrt{\pi}}{\nu (1-\sqrt{\pi})^2} L G\Tilde{G} T\alpha^2 \sqrt{d} +  \frac{4\sqrt{\pi}}{\sqrt{\nu}(1-\sqrt{\pi})^2} G T \alpha \cdot\sigma \sqrt{\frac{N-\tau}{\tau (N-1)}}.
\end{align} 
Hence we obtain
\begin{align}
    & \Bigg[(\alpha(\Tilde{G}_\infty^2+\nu)^{-1/2} - \bigg[2L+3L\bigg(\frac{\beta_1}{1-\beta_1}\bigg)^2\bigg]\frac{4\alpha^2}{\nu} \Bigg] \sum_{t=1}^T \alpha \EE\big[\big\|\nabla f(\xb_t)\big\|_2^2\big] \notag\\
    & \leq \EE\bigg[\Delta f + \frac{\alpha G \Tilde{G} \nu^{-1/2}}{1-\beta_1}\bigg] + \bigg[2L+3L\bigg(\frac{\beta_1}{1-\beta_1}\bigg)^2\bigg] \frac{32T\alpha^2 \Tilde{G}^2}{\nu} + \bigg[2L+3L\bigg(\frac{\beta_1}{1-\beta_1}\bigg)^2\bigg] \frac{4T\alpha^2\sigma^2}{\nu n} \notag\\
    & \qquad + \frac{2\sqrt{\pi}}{\nu(1-\sqrt{\pi})^2} L G\Tilde{G} T\alpha^2 \sqrt{d} +  \frac{4\sqrt{\pi}}{\sqrt{\nu}(1-\sqrt{\pi})^2} G T \alpha \cdot \sigma\sqrt{\frac{N-\tau}{\tau (N-1)}}.
\end{align}
Let $[2L+3L(\frac{\beta_1}{1-\beta_1})^2]\frac{4\alpha^2}{\nu} \leq \frac{\alpha}{2(\Tilde{G}_\infty^2+\nu)^{1/2}}$, we obtain the condition on learning rate $\alpha \leq \frac{\nu}{8(\Tilde{G}_\infty^2+\nu)^{1/2}[2L+3L(\frac{\beta_1}{1-\beta_1})^2]} = \frac{\nu}{4CC_1}$. Then we get
\begin{align}
    \min_{t\in[T]} \EE\big[\big\|\nabla f(\xb_t)\big\|_2^2\big] \leq\frac{1}{T}\sum_{t=1}^T \EE\big[\big\|\nabla f(\xb_t)\big\|_2^2\big] \leq \frac{M_1}{T \alpha} + \frac{M_2}{T} + \alpha M_3 + \frac{\alpha M_4 \sigma^2}{n} + M_5 \sigma \sqrt{\frac{N-\tau}{\tau(N-1)}},
\end{align}
where
\begin{align*}
    M_1 & = C\cdot\Delta f,\quad
    M_2 = \frac{C G \Tilde{G}}{(1-\beta_1)\sqrt{\nu}},\quad
    M_3 = \frac{32CC_1 \Tilde{G}^2}{\nu} + \frac{2\sqrt{\pi}C L G \Tilde{G} \sqrt{d}}{\nu(1-\sqrt{\pi})^2}, \notag\\
    M_4 & = \frac{4CC_1}{\nu}, \quad M_5 = \frac{4\sqrt{\pi} C G}{\sqrt{\nu}(1-\sqrt{\pi})^2},
\end{align*}
and 
\begin{align*}
    \Delta f & = f(\xb_1)- \inf_x f(\xb), \Tilde{G} = C_2G, \Tilde{G}_\infty =  C_2G_\infty, \\
    C & = 2 (\Tilde{G}_\infty^2+\nu)^{1/2}, C_1 = 2L+3L\bigg(\frac{\beta_1}{1-\beta_1}\bigg)^2, 
    C_2 = \frac{(1+\sqrt{\pi})^2}{(1-\sqrt{\pi})^2}.
\end{align*}
After $T$ iterations, it requires 
\begin{align*}
    \frac{M_1}{T\alpha}+\alpha \bigg(M_3+\frac{M_4\sigma^2}{n}\bigg) \leq \frac{\epsilon}{3},\quad \frac{M_2}{T} \leq \frac{\epsilon}{3}, \quad M_5 \sigma \sqrt{\frac{N-\tau}{\tau(N-1)}} \leq \frac{\epsilon}{3}.
\end{align*}
Note that $\frac{n\epsilon}{6nM_3+6M_4\sigma^2} \leq \frac{\epsilon \nu}{32CC_1 \tilde{G}^2} \leq \frac{\nu}{4CC_1} $, therefore, we get
\begin{align}
    \alpha & \leq \frac{n\epsilon}{6nM_3+6M_4\sigma^2}, \quad
    \tau \geq \bigg\lceil \frac{N(3M_5 \sigma)^2}{(N-1)\epsilon^2+ (3M_5 \sigma)^2}\bigg\rceil,\quad
    T \geq \bigg\lceil\frac{36 M_1 M_3}{\epsilon^2} +\frac{36 M_1 M_4 \sigma^2}{n\epsilon^2}+ \frac{3M_2}{\epsilon} \bigg\rceil.
\end{align}
This concludes the proof.

\end{proof}

\section{Supporting Lemmas for Algorithm~\ref{alg:compadam1}}

\begin{proof}[Proof of Lemma~\ref{lm:mVbound-main}]
Since by Assumption~\ref{as:smooth}, $f$ has $G$-bounded stochastic gradient, for any $\xb$ and $\xi$, there is $\|\nabla f(\xb; \xi)\|_2 \leq G$. We have 
\begin{align}
     \|\nabla f(\xb)\|_2 = \|\EE_\xi\nabla f(\xb;\xi)\|_2 \leq \EE_\xi\|\nabla f(\xb;\xi)\|_2 \leq G. \notag
\end{align}

First we prove the upper bound for compressed gradient on $i$-th worker by induction.  For $t=1$, it is obvious that $\big\|\gbh^{(i)}_1\big\|_2\leq \big\|\cC(\gb^{(i)}_1)\big\|_2 \leq \big\|\gb^{(i)}_1\big\|_2 \leq G$.
For $t>1$, suppose we have $\big\|\gbh_{t}^{(i)}\big\|_2 \leq \frac{1+\sqrt{\pi}}{1-\sqrt{\pi}} G=\hat{G}$, our goal is to show that $\big\|\gbh_{t+1}^{(i)}\big\|_2 \leq \frac{1+\sqrt{\pi}}{1-\sqrt{\pi}} G$:
\begin{align}
    \big\|\gbh_{t+1}^{(i)}\big\|_2& = \big\|\gbh_t^{(i)}+\cC(\gb_{t+1}^{(i)}-\gbh_t^{(i)})\big\|_2 \notag\\
    & = \big\|\gbh_t^{(i)}+\cC(\gb_{t+1}^{(i)}-\gbh_t^{(i)})-\gb_{t+1}^{(i)}+\gb_{t+1}^{(i)}\big\|_2 \notag\\
    & \leq \big\|\gbh_t^{(i)}+\cC(\gb_{t+1}^{(i)}-\gbh_t^{(i)})-\gb_{t+1}^{(i)}\big\|_2+\big\|\gb_{t+1}^{(i)}\big\|_2 \notag\\
    & \leq \sqrt{\pi}\big\| \gb_{t+1}^{(i)}-\gbh_t^{(i)}\big\|_2+\big\|\gb_{t+1}^{(i)}\big\|_2 \notag\\
    & \leq \sqrt{\pi}(G+\frac{1+\sqrt{\pi}}{1-\sqrt{\pi}} G)+G\notag\\
    & \leq \frac{1+\sqrt{\pi}}{1-\sqrt{\pi}} G:=\hat{G},
\end{align}
where the first inequality holds due to triangle inequality and the second inequality holds by Assumption~\ref{as:compressor}.

Then we bound the aggregated compressed gradient $\gbh_t$:
\begin{align}
\|\gbh_t\|_2=\Big\| \frac{1}{n} \sum_{i=1}^n\gbh_t^{(i)}  \Big\|_2 \leq \frac{1}{n} \sum_{i=1}^n \big\|\gbh_t^{(i)}  \big\|_2 \leq \hat{G}.
\end{align}

Then we prove the upper bound for compressed aggregated gradient $\gbt$ on the server by induction. Similarly, for $t=1$, we have $\|\gbt_1\|_2\leq \|\cC(\gbh_1)\|_2 \leq \|\gbh_1\|_2 \leq \hat{G}$. For $t>1$, suppose we have $\|\gbt_t\|_2 \leq \frac{1+\sqrt{\pi}}{1-\sqrt{\pi}}\hat{G}$, then we have 

\begin{align}
\|\gbt_{t+1}\|_2 & = \|\gbt_t+\cC(\gbh_{t+1}-\gbt_t)\|_2 \notag\\
& \leq \|\gbt_t+\cC(\gbh_{t+1}-\gbt_t)-\gbh_{t+1}\|_2+ \|\gbh_{t+1}\|_2 \notag\\
& \leq \sqrt{\pi} \|\gbh_{t+1}-\gbt_t\|_2 + \|\gbh_{t+1}\|_2 \notag\\
& \leq \frac{1+\sqrt{\pi}}{1-\sqrt{\pi}} \hat{G}:=\Tilde{G}.
\end{align}
Next we bound $\|\mb_t\|_2$ by induction. We prove this by induction. For $t=0$, we have $\|\mb_0\|_2 = 0 \leq \Tilde{G}$. For $t \geq 1$, suppose that $\|\mb_t\|_2 \leq \Tilde{G}$, then for $\mb_{t+1}$, we have
\begin{align*}
    \|\mb_{t+1}\|_2 &=  \|\beta_1\mb_{t}+(1-\beta_1)\gbt_{t+1}\|_2\\ 
    &\leq \beta_1\|\mb_t\|_2  + (1-\beta_1)\|\gbt_{t+1}\|_2 \\
    &\leq \beta_1\Tilde{G} + (1-\beta_1) \Tilde{G} \\
    &= \Tilde{G}.\notag
\end{align*}
Thus, for any $t \geq 0$, we have $\|\mb_t\|_2 \leq \Tilde{G}$. 

Finally we bound $\|\hat{\vb}_t\|_2$ again by induction. For $t=1$, we have $\|\vb_0\|_2 = \|\hat{\vb}_0\|_2=0\leq \Tilde{G}^2$. For $t>1$, suppose we have $\|\hat{\vb}_t\|_2 \leq \Tilde{G}^2$, then for $\vb_{t+1}$,
\begin{align*}
     \|\vb_{t+1}\|_2&= \|\beta_2\vb_t+(1-\beta_2)\gbt_{t+1}^2\|_2 \\
    &\leq \beta_2\|\vb_t\|_2+(1-\beta_2)\|\gbt_{t+1}^2\|_2 \\
    &\leq \beta_2 \Tilde{G}^2+(1-\beta_2)\Tilde{G}^2\\
    &=\Tilde{G}^2,\notag
\end{align*}
and by definition, we have $\|\hat\vb_{t+1}\|_2 = \max\{\|\hat\vb_t\|_2, \|\vb_{t+1}\|_2\} \leq \Tilde{G}^2$. 
Thus, for any $t \geq 0$, we have $\|\hat{\vb}_t\|_2 \leq \Tilde{G}^2$.
The proof of $\ell_\infty$-boundedness is similar to the proof of $\ell_2$-boundedness. 
This concludes the proof. 
\end{proof}

\begin{proof}[Proof of Lemma~\ref{lm:mini-batch-var}] 

By definition, we have 
\begin{align*}
    \EE [\|\gb_t^{(i)} - \nabla f_i(\xb_t)\|_2^2] = \EE\bigg[\bigg\|\frac{1}{\tau}\sum_{j \in \cS_\tau^i} \gb_t^{(i,j)} - \nabla f_i(\xb_t) \bigg\|_2^2 \bigg],
\end{align*}
where $\cS_\tau^i$ is a set of batch index with $\tau$ samples on the $i$-th worker, then for $\forall j \in \cS_\tau^i$, we have 
\begin{align}\label{eq:gi-fi}
    & \qquad \EE\bigg[\bigg\|\frac{1}{\tau}\sum_{j \in \cS_\tau^i} \gb_t^{(i,j)} - \nabla f_i(\xb_t) \bigg\|_2^2 \bigg] \notag\\
    & = \frac{1}{\tau^2} \EE\bigg[\sum_{j,j'\in \cS_\tau^i}[\gb_t^{(i,j)} - \nabla f_i(\xb_t)]^\top [\gb_t^{(i,j')} - \nabla f_i(\xb_t)] \bigg] \notag\\
    & = \frac{1}{\tau^2} \EE\bigg[\sum_{j\neq j' \in \cS_\tau^i}[\gb_t^{(i,j)} - \nabla f_i(\xb_t)]^\top [\gb_t^{(i,j')} - \nabla f_i(\xb_t)] \bigg] + \frac{1}{\tau} \EE\big[\big\|\gb_t^{(i,j)} - \nabla f_i(\xb_t)\big\|_2^2\big] \notag\\
    & = \frac{\tau-1}{\tau N(N-1)} \bigg[\sum_{j\neq j'\in[N] }[\gb_t^{(i,j)} - \nabla f_i(\xb_t)]^\top [\gb_t^{(i,j')} - \nabla f_i(\xb_t)] \bigg] + \frac{1}{\tau} \EE\big[\big\|\gb_t^{(i,j)} - \nabla f_i(\xb_t)\big\|_2^2\big] \notag\\
    & = \frac{\tau-1}{\tau N(N-1)} \bigg[\sum_{j, j'\in[N]}[\gb_t^{(i,j)} - \nabla f_i(\xb_t)]^\top [\gb_t^{(i,j')} - \nabla f_i(\xb_t)] \bigg]\notag\\
    & \qquad  -\frac{\tau-1}{\tau(N-1)}\EE\big[\big\|\gb_t^{(i,j)} - \nabla f_i(\xb_t)\big\|_2^2\big] + \frac{1}{\tau} \EE\big[\big\|\gb_t^{(i,j)} - \nabla f_i(\xb_t)\big\|_2^2\big] \notag\\
    & = \frac{N-\tau}{\tau (N-1)}\EE\big[\big\|\gb_t^{(i,j)} - \nabla f_i(\xb_t)\big\|_2^2\big] \notag\\
    & \leq \frac{N-\tau}{\tau (N-1)} \sigma^2,
\end{align}
where the second equality holds by the sampling strategy, the third one follows by $\PP\{j,j' \in \cS_\tau^i \}= \frac{\tau(\tau-1)}{N(N-1)}$. The inequality holds due to Assumption \ref{as:bound-variance}. Therefore, we have 
\begin{align*}
    \EE\bigg[\bigg\|\frac{1}{\tau}\sum_{j \in \cS_\tau^i} \gb_t^{(i,j)} - \nabla f_i(\xb_t) \bigg\|_2^2 \bigg] \leq \frac{N-\tau}{\tau (N-1)} \sigma^2.
\end{align*}
This concludes the proof.

\end{proof}

\begin{proof}[Proof of Lemma~\ref{lm:g-gh-gt}]

For the squared norm of compressed gradient, we have
\begin{align}
    \EE[\|\gbt_t\|_2^2] & = \EE[\|\gbt_t - \gbh_t + \gbh_t - \gb_t + \gb_t - \nabla f(\xb_t) + \nabla f(\xb_t)\|_2^2] \notag\\
    & \leq 4\EE[\|\gb_t - \nabla f(\xb_t)\|_2^2] + 4\EE[\|\gbt_t - \gbh_t\|_2^2] + 4\EE[\|\gbh_t - \gb_t\|_2^2] + 4\EE[\|\nabla f(\xb_t)\|_2^2].
\end{align}
For the stochastic variance, we have
\begin{align} \label{eq:Eg-f}
    \EE[\|\gb_t - \nabla f(\xb_t)\|_2^2] & = \EE\bigg[\bigg\|\frac{1}{n}\sum_{i=1}^n [\gb_t^{(i)} - \nabla f_i(\xb_t)] \bigg\|_2^2\bigg]\notag\\
    & = \frac{1}{n^2} \sum_{i=1}^n \EE [\|\gb_t^{(i)} - \nabla f_i(\xb_t)\|_2^2] \notag\\
    & = \frac{1}{n} \EE [\|\gb_t^{(i)} - \nabla f_i(\xb_t)\|_2^2] \notag\\
    & \leq \frac{\sigma^2(N-\tau)}{n\tau(N-1)} \notag\\
    & \leq \frac{\sigma^2}{n},
\end{align}
where the first inequality holds by Lemma \ref{lm:mini-batch-var}, and the second one holds by $1\leq \tau\leq N$. For the squared norm of gradient compression error, we have 
\begin{align}
    \EE[\|\gbh_t-\gb_t\|_2^2] & \leq \EE[\|\gbh_t\|_2^2] + \EE[\|\gb_t\|_2^2] \notag\\
    & \leq 2\hat{G}^2 + 2G^2,
\end{align}
where the second inequality follows by Lemma \ref{lm:mVbound-main}. Similarly, for the squared norm of gradient compression error on the server side, we have
\begin{align}
    \EE[\|\gbt_t-\gbh_t\|_2^2] & \leq \EE[\|\gbt_t\|_2^2] + \EE[\|\gbh_t\|_2^2] \notag\\
    & \leq 2\hat{G}^2 + 2\Tilde{G}^2,
\end{align}
where the second inequality follows by Lemma \ref{lm:mVbound-main}.
Summing over $t=1,...,T$, we have
\begin{align}
    \sum_{t=1}^T \EE[\|\gbt_t\|_2^2] & \leq 4T(2\hat{G}^2 + 2\Tilde{G}^2)+4T(2\hat{G}^2 + 2G^2) + \frac{4T\sigma^2}{n} + 4\sum_{t=1}^T\EE[\|\nabla f(\xb_t)\|_2^2] \notag\\
    & \leq 32T\Tilde{G}^2+ \frac{4T\sigma^2}{n} + 4\sum_{t=1}^T\EE[\|\nabla f(\xb_t)\|_2^2],
\end{align}
where the second inequality follows by $G\leq\hat{G}\leq\tilde{G}$. By the updating rule, we have
\begin{align}
    \EE[\|\mb_{t}\|^{2}] 
    & = \EE\bigg[\|(1-\beta_1) \sum_{\tau=1}^{t} \beta_{1}^{t-\tau} \gbt_{\tau}\|^{2}\bigg] \notag\\
    & \leq (1-\beta_1)^2 \sum_{i=1}^{d} \EE\bigg[\bigg(\sum_{\tau=1}^{t} \beta_{1}^{t-\tau} \gbt_{\tau, i}\bigg)^{2}\bigg] \notag\\
    & \leq (1-\beta_1)^{2} \sum_{i=1}^{d} \EE\bigg[\bigg(\sum_{\tau=1}^{t} \beta_{1}^{t-\tau}\bigg)\bigg(\sum_{\tau=1}^{t} \beta_{1}^{t-\tau} \gbt_{\tau, i}^{2}\bigg)\bigg] \notag \\
    & \leq (1-\beta_1) \sum_{\tau=1}^{t} \beta_{1}^{t-\tau} \EE[\|\gbt_{\tau}\|^{2}] \notag\\
    & \leq 32\Tilde{G}^2 + \frac{4\sigma^2}{n} + 4(1-\beta_1)\sum_{\tau=1}^t \beta_1^{t-\tau}\EE[\|\nabla f(\xb_t)\|_2^2],
\end{align}
where the second inequality holds by applying Cauchy-Schwarz inequality, and the third inequality holds by summation of series. 
Hence summing over $t=1,\cdots, T$, we have
\begin{align}
    \sum_{t=1}^T \EE\|\mb_t\|^2 & \leq 32 T 
    \Tilde{G}^2 + \frac{4T\sigma^2}{n} + 4\sum_{t=1}^T\EE[\|\nabla f(\xb_t)\|_2^2].
\end{align}
Therefore, with the condition $\alpha_t = \alpha$ in Lemma \ref{lm:g-gh-gt}, we have
\begin{align}
    & \sum_{t=1}^T \alpha_t^2 \EE\big[ \big\|\Hat{\Vb}_t^{-1/2}\mb_t\big\|_2^2\big] 
    \leq \frac{\alpha^2}{\nu} \sum_{t=1}^T \EE[\|\mb_t\|_2^2] \leq \frac{32T\alpha^2 \Tilde{G}^2}{\nu} + \frac{4T\alpha^2\sigma^2}{\nu n} + \frac{4\alpha^2}{\nu}\sum_{t=1}^T\EE[\|\nabla f(\xb_t)\|_2^2].
\end{align}
We also have
\begin{align}
    & \sum_{t=1}^T \alpha_t^2 \EE\big[ \big\|\Hat{\Vb}_t^{-1/2}\gbt_t\big\|_2^2\big] 
     \leq \frac{\alpha^2}{\nu} \sum_{t=1}^T \EE[\|\gbt_t\|_2^2] \leq \frac{32T\alpha^2 \Tilde{G}^2}{\nu} + \frac{4T\alpha^2\sigma^2}{\nu n} + \frac{4\alpha^2}{\nu}\sum_{t=1}^T\EE[\|\nabla f(\xb_t)\|_2^2].
\end{align}

\end{proof}

\begin{proof}[Proof of Lemma~\ref{lm:gb-gbh}]
Considering on the $i$-th worker, we have
\begin{align}
\EE [\|\gbh_t^{(i)}-\gb_t^{(i)}\|_2] 
& = \EE [\|\gbh_{t-1}^{(i)}+ \cC(\gb_t^{(i)} - \gbh_{t-1}^{(i)})-\gb_t^{(i)}\|_2] \notag\\
& \leq \sqrt{\pi} \EE [\|\gb_t^{(i)}-\gbh_{t-1}^{(i)}\|_2] \notag\\
& \leq \sqrt{\pi} \EE [\|\gbh_{t-1}^{(i)}-\gb_{t-1}^{(i)}\|_2]  +\sqrt{\pi} \EE [\|\gb_t^{(i)}-\gb_{t-1}^{(i)}\|_2],  \label{eq:gbhi-nfi}
\end{align}
where the first inequality holds because of Assumption~\ref{as:compressor}.
For the compression error between the local stochastic gradient and the gradient of loss function, by Lemma \ref{lm:mini-batch-var} we have
\begin{align*}
\EE [\|\gb_{t}^{(i)}- \nabla f_i(\xb_{t}) \|_2^2] & = \frac{N-\tau}{\tau(N-1)}\sigma^2.
\end{align*} 
Since for all variables $\mathbf{X}$, there is $\big[\EE \mathbf{X} \big]^2 \leq \EE \mathbf{X}^2$, we have
\begin{align*}
    \EE [\|\gb_{t}^{(i)}- \nabla f_i(\xb_{t}) \|_2] \leq \big[\EE \|\gb_{t}^{(i)}- \nabla f_i(\xb_{t}) \|_2^2\big]^{\frac{1}{2}} \leq \sigma \sqrt{\frac{N-\tau}{\tau (N-1)}} .
\end{align*}
Then the expectation of gradient compression error on all workers is
\begin{align}\label{eq:gbh-gb}
    \EE [\|\gbh_t-\gb_t\|_2] 
    & =\EE \bigg[\bigg\|\frac{1}{n} \sum_{i=1}^n \gbh_t^{(i)} -\frac{1}{n} \sum_{i=1}^n \gb_t^{(i)}\bigg\|_2\bigg] \notag\\
    & \leq \frac{1}{n}\sum_{i=1}^n \EE_{\xi^{(i)}} [\|\gbh_t^{(i)}-\gb_t^{(i)}\|_2] \notag\\
    & \leq \sqrt{\pi} \frac{1}{n}\sum_{i=1}^n \EE_{\xi^{(i)}} [\|\gbh_{t-1}^{(i)}-\gb_{t-1}^{(i)}\|_2] + \sqrt{\pi}  \frac{1}{n}\sum_{i=1}^n \EE_{\xi^{(i)}} [\|\gb_{t}^{(i)}- \gb_{t-1}^{(i)} \|_2] \notag \\
    & \leq \sqrt{\pi} \frac{1}{n}\sum_{i=1}^n \EE [\|\gbh_{t-1}^{(i)}-\gb_{t-1}^{(i)}\|_2]+ \sqrt{\pi}  \frac{1}{n}\sum_{i=1}^n \EE [\|\nabla f_i(\xb_t)- \nabla f_i(\xb_{t-1})\|_2]  \notag\\
    & \qquad + \sqrt{\pi} \frac{1}{n}\sum_{i=1}^n \EE [\|\gb_{t}^{(i)}- \nabla f_i(\xb_t) \|_2] + \sqrt{\pi}  \frac{1}{n}\sum_{i=1}^n \EE [\|\gb_{t-1}^{(i)}- \nabla f_i(\xb_{t-1}) \|_2] \notag \\
    & \leq \sqrt{\pi} \frac{1}{n}\sum_{i=1}^n \EE [\|\gbh_{t-1}^{(i)}-\gb_{t-1}^{(i)}\|_2]  +\sqrt{\pi} L \|\xb_t-\xb_{t-1}\|_2 + 2\sqrt{\pi}\sigma \sqrt{\frac{N-\tau}{\tau (N-1)}},
\end{align}
where the first inequality holds by extended triangle inequality, the second inequality holds by applying \eqref{eq:gbhi-nfi} and the last inequality holds by $L$-smoothness. Telescoping the right hand side of \eqref{eq:gbh-gb}, we have 
\begin{align}
    \EE [\|\gbh_t-\gb_t\|_2] 
    & \leq \sqrt{\pi} \frac{1}{n}\sum_{i=1}^n \EE [\|\gbh_{t-1}^{(i)}-\gb_{t-1}^{(i)}\|_2]  +\sqrt{\pi} L \|\xb_t-\xb_{t-1}\|_2 + 2\sqrt{\pi}\sigma \sqrt{\frac{N-\tau}{\tau (N-1)}} \notag\\
    & \leq L \sum_{j=1}^t (\sqrt{\pi})^{t+1-j}  \|\xb_j-\xb_{j-1}\|_2 +2\sigma \sqrt{\frac{N-\tau}{\tau (N-1)}} \sum_{j=1}^t (\sqrt{\pi})^{t+1-j}  \notag\\
    & = L \sum_{j=1}^t (\sqrt{\pi})^{t+1-j}  \big\|\alpha_{j-1} \Vbh_{j-1}^{-1/2} \mb_{j-1}\big\|_2 +2\sigma \sqrt{\frac{N-\tau}{\tau (N-1)}} \sum_{j=1}^t (\sqrt{\pi})^{t+1-j} \notag\\
    & \leq L\sum_{j=1}^t (\sqrt{\pi})^{t+1-j}  \big\| \alpha_{j-1} \Vbh_{j-1}^{-1/2}\big\|_2 \big\|\mb_{j-1}\big\|_2 + 2\sigma \sqrt{\frac{N-\tau}{\tau (N-1)}} \sum_{j=1}^t (\sqrt{\pi})^{t+1-j}  \notag\\
    & \leq L \Tilde{G} \sum_{j=2}^t (\sqrt{\pi})^{t+1-j}  \big\| \alpha_{j-1}\Vbh_{j-1}^{-1/2}\big\|_2 + 2\sigma \sqrt{\frac{N-\tau}{\tau (N-1)}} \sum_{j=1}^t (\sqrt{\pi})^{t+1-j} , \label{g=theta*v}
\end{align}

where the second inequality holds by $L$-smooth and the recursion \eqref{eq:gbh-gb}. The third inequality holds by taking iteration of $\|\gbh_t -\gb_t\|$. The forth inequality holds by the update rule of $\xb$ and for matrix $\mathbf{A}$ and vector $\xb$, we have $\|\mathbf{A}\xb\|_2\leq\|\mathbf{A}\|_2\|\xb\|_2$.
The last inequality holds because we set the initial momentum $\mb_0=0$ and by Lemma~\ref{lm:mVbound-main}, there is an upper bound $\|\mb_t\|_2\leq \Tilde{G}$. 

Since $\alpha_{j-1} \Vbh_{j-1}^{-1/2} \succeq \alpha_{j} \Vbh_{j}^{-1/2} \succeq 0$, we have 
\begin{align}
    \EE [\|\gbh_t-\gb_t\|_2] & \leq L \Tilde{G} \sum_{j=2}^t (\sqrt{\pi})^{t+1-j}  \big\| \alpha_{1}\Vbh_{1}^{-1/2}\big\|_2 + 2\sigma \sqrt{\frac{N-\tau}{\tau (N-1)}} \sum_{j=1}^t (\sqrt{\pi})^{t+1-j}\notag\\
    & \leq  \frac{\sqrt{\pi}}{1-\sqrt{\pi}} L \Tilde{G}\alpha_1 \sqrt{d} \nu^{-1/2}+\frac{2\sigma\sqrt{\pi}}{1-\sqrt{\pi}} \sqrt{\frac{N-\tau}{n\tau (N-1)}},
\end{align}
where the last inequality holds because $\big\|\Vbh_{1}^{-1/2}\big\|_2 \leq \sqrt{d} \nu^{-1/2}$. This completes the proof.
\end{proof}

\begin{proof}[Proof of Lemma~\ref{lm:gbh-gbt}] 
By Assumption~\ref{as:compressor}, we have 
\begin{align}\label{gbt-gbh1}
    \EE[\|\gbt_t-\gbh_t\|_2] 
    & \leq \EE [\|\gbt_{t-1}+\cC(\gbh_t-\gbt_{t-1})-\gbh_{t}\|_2] \notag\\
    &  \leq \sqrt{\pi} \EE[\|\gbh_t-\gbt_{t-1}\|_2] \notag\\
    & \leq \sqrt{\pi} \EE [\|\gbh_{t-1}-\gbt_{t-1}\|_2] + \sqrt{\pi} \EE[\|\gbh_t-\gbh_{t-1}\|_2] \notag\\
    & \leq \sum_{j=1}^t (\sqrt{\pi})^{t+1-j} \cdot \EE[\|\gbh_j-\gbh_{j-1}\|_2],
\end{align}
the last inequality holds by telescoping the right hand side of the previous inequality, with the similar approach in \eqref{g=theta*v}, we have
\begin{align}\label{gbhj-gbhj-1}
    \EE [\|\gbh_j-\gbh_{j-1}\|_2] & = \EE \bigg[\bigg\| \frac{1}{n}\sum_{i=1}^n \gbh_j^{(i)}-\frac{1}{n}\sum_{i=1}^n \gbh_{j-1}^{(i)}\bigg\|_2\bigg] \notag\\
    & = \EE \bigg[\bigg\|\frac{1}{n}\sum_{i=1}^n \cC(\gb_{j}^{(i)}-\gbh_{j-1}^{(i)}) \bigg\|_2\bigg] \notag\\
    & \leq \frac{1}{n}\sum_{i=1}^n \EE \big[\big\| \cC(\gb_{j}^{(i)}-\gbh_{j-1}^{(i)}) \big\|_2\big] \notag\\
    & \leq \frac{1+\sqrt{\pi}}{n}\sum_{i=1}^n \EE \big[\big\| \gb_{j}^{(i)}-\gbh_{j-1}^{(i)}\big\|_2\big] \notag\\
    & \stackrel{(a)}{\leq} \frac{1+\sqrt{\pi}}{n}\sum_{i=1}^n \EE \big[\big\|\gbh_{j-1}^{(i)}-\gb_{j-1}^{(i)}\big\|_2\big] +  \frac{1+\sqrt{\pi}}{n}\sum_{i=1}^n \EE \big[\big\| \gb_{j-1}^{(i)}-\gb_{j}^{(i)}\big\|_2\big],
\end{align}
where we apply Jensen's inequality to get the first inequality. The second inequality holds due to Lemma~\ref{lm: compressor}, the third inequality holds by triangle inequality, then similar to previous proof, we have
\begin{align}
    \EE[\|\gbh_j-\gbh_{j-1}\|_2] 
    & \leq \frac{1+\sqrt{\pi}}{n}\sum_{i=1}^n \EE \big[\big\|\gbh_{j-1}^{(i)}-\gb_{j-1}^{(i)}\big\|_2\big]+ \frac{1+\sqrt{\pi}}{n}\sum_{i=1}^n \EE \big[\big\| \nabla f_i(\xb_{j-1})-\nabla f_i(\xb_j)\big\|_2\big]\notag\\
    & \qquad + \frac{1+\sqrt{\pi}}{n}\sum_{i=1}^n \EE \big[\big\| \gb_{j}^{(i)}-\nabla f_i(\xb_j)\big\|_2\big] + \frac{1+\sqrt{\pi}}{n}\sum_{i=1}^n \EE \big\| \gb_{j-1}^{(i)}-\nabla f_i(\xb_{j-1})\big\|_2\notag\\
    & \leq \frac{(1+\sqrt{\pi})(1-\sqrt{\pi}^j)}{1-\sqrt{\pi}}L \Tilde{G} \alpha_1  \big\|\Vbh_{1}^{-1/2}\big\|_2+ \frac{(1+\sqrt{\pi})(1-\sqrt{\pi}^j)}{1-\sqrt{\pi}}\cdot 2\sigma \sqrt{\frac{N-\tau}{\tau (N-1)}} \notag\\
    & \leq \frac{1+\sqrt{\pi}}{1-\sqrt{\pi}}L \Tilde{G} \alpha_1  \big\|\Vbh_{1}^{-1/2}\big\|_2+ \frac{1+\sqrt{\pi}}{1-\sqrt{\pi}} \cdot 2\sigma \sqrt{\frac{N-\tau}{\tau (N-1)}},
\end{align}
where inequalities hold by proof of Lemma~\ref{lm:gb-gbh}.
Substituting \eqref{gbhj-gbhj-1} into \eqref{gbt-gbh1}
\begin{align} \label{gbt-gbh2}
    \EE [\|\gbt_t-\gbh_t\|_2] & \leq  \frac{1+\sqrt{\pi}}{1-\sqrt{\pi}} L \Tilde{G} \alpha_1 \cdot \big\|\Vbh_{1}^{-1/2}\big\|_2 \sum_{j=1}^t (\sqrt{\pi})^{t+1-j}+ \frac{1+\sqrt{\pi}}{1-\sqrt{\pi}} \cdot 2\sigma\sqrt{\frac{N-\tau}{\tau (N-1)}} \sum_{j=1}^t (\sqrt{\pi})^{t+1-j} \notag\\
    & \leq \frac{\sqrt{\pi}(1+\sqrt{\pi})}{(1-\sqrt{\pi})^2}  L \Tilde{G} \alpha_1 \big\|\Vbh_{1}^{-1/2}\big\|_2 + \frac{\sqrt{\pi}(1+\sqrt{\pi})}{(1-\sqrt{\pi})^2} \cdot 2\sigma\sqrt{\frac{N-\tau}{\tau (N-1)}} \sum_{j=1}^t (\sqrt{\pi})^{t+1-j} \notag\\
    & \leq \frac{\sqrt{\pi}(1+\sqrt{\pi})}{(1-\sqrt{\pi})^2}  L \Tilde{G} \alpha_1 \sqrt{d} \nu^{-1/2} + \frac{\sqrt{\pi}(1+\sqrt{\pi})}{(1-\sqrt{\pi})^2}\cdot 2\sigma \sqrt{\frac{N-\tau}{\tau (N-1)}} .
\end{align}
where the last inequality holds because $\big\|\Vbh_{1}^{-1/2}\big\|_2 \leq \sqrt{d} \nu^{-1/2}$. This concludes the proof.
\end{proof}

\section{Dependency on the Compression Constant}

We track the impact of the compressor related parameter $\pi$. Recall in Assumption \ref{as:compressor}, we assume we have compressors satisfy $\EE_\cC[\|\cC(\xb)-\xb\|_2^2] \leq \pi \xb^2 , \forall \xb \in \RR^d$. Note that $\pi = 0$ leads to $\cC(\xb) = \xb$. In fact, a stronger compression means a larger $\pi$, i.e., $\pi \to 1$, may leads to worse convergence. For our convergence bound
\begin{align}\label{eq:thm-appendix-D}
    \min_{t\in[T]} \EE\big[\big\|\nabla f(\xb_t)\big\|_2^2\big] &\leq\frac{1}{T}\sum_{t=1}^T \EE\big[\big\|\nabla f(\xb_t)\big\|_2^2\big] \notag\\
    &\leq \frac{M_1}{T \alpha} + \frac{M_2}{T} + \alpha M_3 + \frac{\alpha M_4 \sigma^2}{n} + M_5 \sigma \sqrt{\frac{N-\tau}{\tau(N-1)}},
\end{align}
where
\begin{align*}
    M_1 & = C\cdot\Delta f,\quad
    M_2 = \frac{C G \Tilde{G}}{(1-\beta_1)\sqrt{\nu}},\quad
    M_3 = \frac{32CC_1 \Tilde{G}^2}{\nu} + \frac{2\sqrt{\pi}C L G \Tilde{G} \sqrt{d}}{\nu(1-\sqrt{\pi})^2}, \quad\\
    M_4 &= \frac{4CC_1}{\nu}, \quad M_5 = \frac{4\sqrt{\pi} C G}{\sqrt{\nu}(1-\sqrt{\pi})^2},
\end{align*}
and 
\begin{align*}
    \Delta f & = f(\xb_1)- \inf_x f(\xb), \Tilde{G} = C_2G, \Tilde{G}_\infty =  C_2G_\infty, \\
    C & = 2 (\Tilde{G}_\infty^2+\nu)^{1/2}, C_1 = 2L+3L\bigg(\frac{\beta_1}{1-\beta_1}\bigg)^2, 
    C_2 = \frac{(1+\sqrt{\pi})^2}{(1-\sqrt{\pi})^2}.
\end{align*}
Following by Lemma \ref{lm:mVbound-main}, as $\pi \to 1$, we conclude the series of compressed gradient bounds $\hat{G}$, $\tilde{G}$, $\hat{G}_\infty$, $\tilde{G}_\infty$ as follows
\begin{align*}
    \hat{G} & = \frac{1+\sqrt{\pi}}{1-\sqrt{\pi}}G = O\bigg(\frac{G}{1-\pi} \bigg), \quad \hat{G}_\infty = \frac{1+\sqrt{\pi}}{1-\sqrt{\pi}}G_\infty = O\bigg(\frac{G_\infty}{1-\pi} \bigg),\\
    \tilde{G} & = C_2 G = \frac{(1+\sqrt{\pi})^2}{(1-\sqrt{\pi})^2} G = O\bigg(\frac{G}{(1-\pi)^2} \bigg),
    \tilde{G}_\infty = C_2 G_\infty = \frac{(1+\sqrt{\pi})^2}{(1-\sqrt{\pi})^2} G_\infty = O\bigg(\frac{G_\infty}{(1-\pi)^2} \bigg),  \\
    C & = 2 (\Tilde{G}_\infty^2+\nu)^{1/2} = O(\tilde{G}_\infty) = O\bigg(\frac{G_\infty}{(1-\pi)^2} \bigg).
\end{align*}

\begin{table}[ht] 
    \centering
    \begin{tabular}{c|c}
        \toprule
        Const & Order of $\pi$ \\
        \midrule
        $M_1$ &  $O\left(\Delta f G_\infty(1-\pi)^{-2}\right)$ \\
        $M_2$ &   $O\left( G^2 G_\infty(1-\pi)^{-4} \right)$\\
        $M_3$ & $O\left(G^2G_\infty (1-\pi)^{-6}+L \sqrt{d} G^2 G_\infty(1-\pi)^{-6} \right)$\\
        $M_4$ & $O\left(G_\infty(1-\pi)^{-2}\right)$ \\
        $M_5$ &  $O\left(G G_\infty(1-\pi)^{-4} \right)$\\
        \hline
        $T$ & $O\left(M_1 M_3 \epsilon^{-2} + M_1 M_4\sigma^2 n^{-1} \epsilon^{-2} + M_2\epsilon^{-1} \right)$\\
        \bottomrule
    \end{tabular}
    \caption{Theorem constants with dependency on $\pi$.}
    \label{tab:pi-dependency}
\end{table}
Table \ref{tab:pi-dependency} summarize the dependency of the theorem constant on $\pi$. Note that here the total number of iterates $T$ scales as $(1-\pi)^{-8}$, which is higher than that of distributed SGD \citep{philippenko2020artemis, philippenko2021preserved}. This is mainly due to the fact that the convergence of adaptive gradient methods and the theoretical proof of adaptive gradient methods rely heavily on the bounded gradient assumption. While for communication-compressed adaptive methods, the gradient bound for compressed gradient $\gbh$ and $\gbt$ is already dependent on $\pi$ as shown in the above derivations. In fact, our dependency on gradient bounds matches the existing works studying vanilla adaptive gradient methods \citep{1808.05671, chen2018on}, which suggests that our bound is indeed tight.

Note that in practice, the high dependency on $\pi$ does not make much difference. In fact, the actual $\pi$ computed using real compressor is in a constant order. For example, in our ResNet-18 experiment with scaled sign compressor, the actual $\pi$ ranges from $[0.597, 0.713]$ during the entire training procedure, which is indeed a reasonable constant. 
\clearpage

\section{Additional Experiments}
In this section, we present the additional experiments of our proposed method. In Section \ref{subsec:ncvx-topk}, we provide experiments using the nonconvex logistic regression with Top-$k$ compressor. In Section \ref{subsec:dl-more}, we present additional deep learning experiments on other popular state-of-the-art model architectures.

\subsection{Nonconvex Logistic Regression Based on Top-k Compressor}\label{subsec:ncvx-topk}
The nonconvex logistic regression problem setting follows exactly the same as in Section \ref{subsec:ncvx-sign}. Here we apply the proposed Markov compression sequence based on the Top-$k$ compressor, where $k=1$ for the nonconvex logistic regression with $d=300$ parameters.

\begin{figure*}[ht]
\centering
\subfigure{\includegraphics[width=1.0\textwidth]{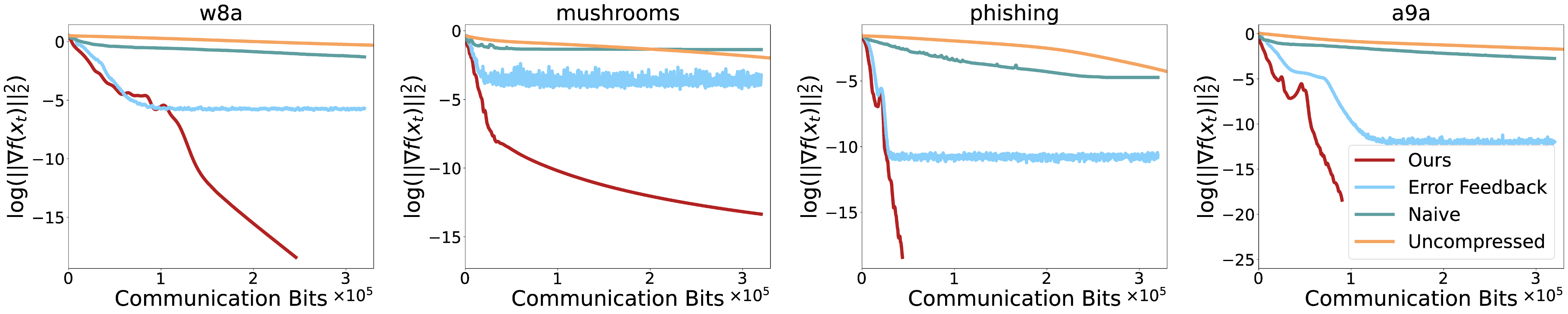}}
\subfigure{\includegraphics[width=1.0\textwidth]{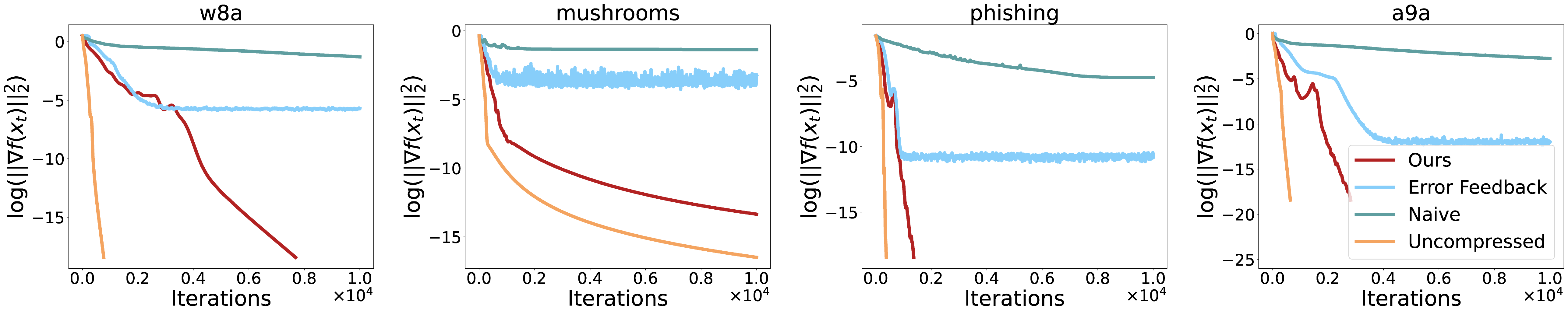}}
\caption{Gradient norm comparison of different compression strategies on nonconvex logistic regression trained by AMSGrad with Top-$1$ based Markov compression sequence. The upper row is the norm convergence with respect to the communication cost, and the lower row is with respect to the training iteration.}
 \label{fig:lr_topk_iter}
\end{figure*}

Figure \ref{fig:lr_topk_iter} shows that our proposed CD-Adam achieved the best performances for all four datasets against the other three compression strategies. Specifically, we can observe that CD-Adam achieves a much better result than AMSGrad with error feedback and naive compression in terms of both communication efficiency and the final performances. Furthermore, compared with uncompressed AMSGrad, our communication efficiency is greatly improved while maintaining similar final performances.

\subsection{Deep Learning Experiments on Image Classification}\label{subsec:dl-more}
In this part, we provide additional deep learning experiments on other state-of-the-art model architectures. The problem setting follows exactly the same as in Section \ref{subsec:dl}. In the following, we show that our proposed algorithm enjoys a fast convergence speed with high accuracy and less communication cost compared with EF21 and 1-bit Adam.

Figure \ref{fig:resnet_full}, \ref{fig:vgg_bits} and \ref{fig:widern_bits} compare the performance with respect to the communication cost using different model architectures. The experiments show that our proposed method obtains better training/testing loss and accuracy, especially much better training/testing gradient norm on various CNN model architectures, including ResNet, VGG, and WideResNet. Notice that under this comparison, EF21, which is also equipped with Markov compression sequences on SGD, performs worse than our method, especially in later training stages. While for 1-bit Adam, since it needs an uncompressed warm-up procedure for a few epochs, the communication efficiency is much worse compared with CD-Adam and EF21. 

Figure \ref{fig:resnet_epoch}, \ref{fig:vgg_epoch} and \ref{fig:widern_epoch} compare the performance with respect to the iteration epochs using different model architectures. Compared with EF21, our proposed method obtains a much better convergence speed in terms of training/testing loss, accuracy, and the gradient norm. Our method also shows a better performance than 1-bit Adam in terms of the final testing accuracy and gradient norm. 

\begin{figure*}[ht]
 \centering
  \includegraphics[width=1.0\textwidth]{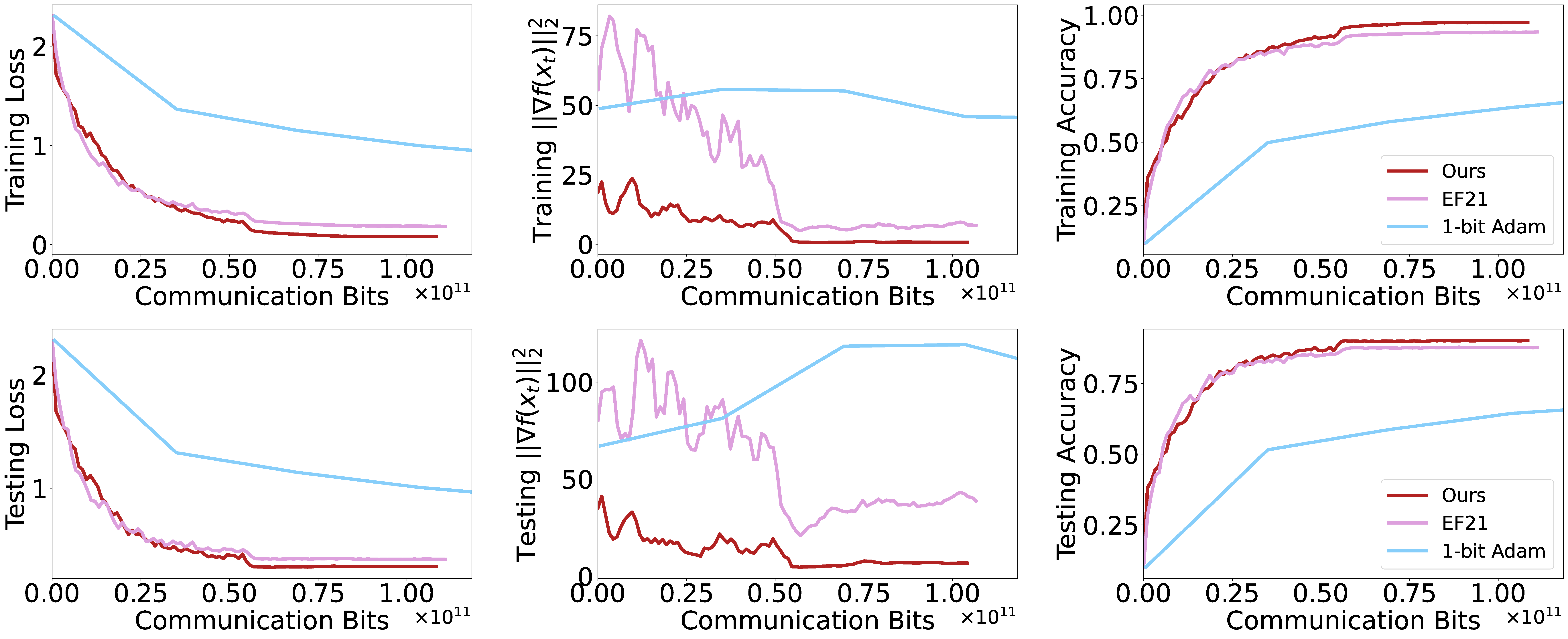}
  \caption{Comparison on training/testing loss, gradient norm and accuracy with respect to communication cost among our proposed method and the baselines when training \texttt{ResNet-18} model on \texttt{CIFAR10}.}
  \label{fig:resnet_full}
\end{figure*}

Figure \ref{fig:resnet_full} is the continuous of Figure \ref{fig:resnet_showcase}. It shows that our proposed CD-Adam achieves better performance than EF21 and 1-bit Adam under the same communication cost. 

\begin{figure*}[ht]
 \centering
  \includegraphics[width=1.0\textwidth]{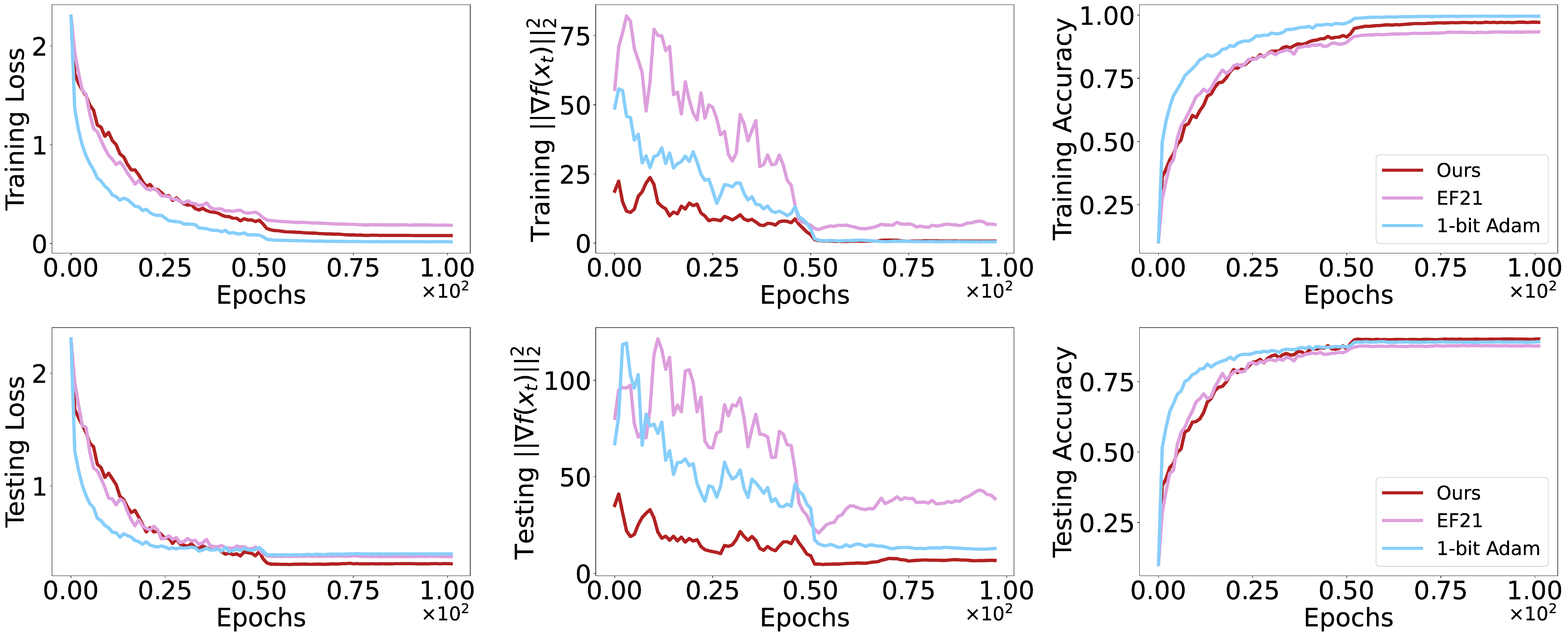}
  \caption{Comparison on training/testing loss, gradient norm and accuracy with respect to epochs among our proposed method and the baselines when training \texttt{ResNet-18} model on \texttt{CIFAR10}.}
  \label{fig:resnet_epoch}
\end{figure*}

In Figure \ref{fig:resnet_epoch}, the left two plots show that at the early stage of the training process, CD-Adam and EF21 obtain a similar speed of reducing the training loss, while 1-bit Adam is quite faster due to the warm-up process. At later stages, CD-Adam shows a clear advantage compared with EF21. Our method also shows better performance on testing loss than 1-bit Adam. From the middle two plots in Figure \ref{fig:resnet_epoch}, we can observe that our proposed CD-Adam obtains a smaller gradient norm than EF21 and 1-bit Adam under the same epochs when training and testing. The right two plots show that our proposed method achieves high accuracy on both the training and testing sets. At the early stage of the training process, CD-Adam and EF21 obtain a similar speed of increasing the accuracy, while 1-bit Adam is faster due to the warm-up process. At later stages, CD-Adam shows a clear advantage compared with EF21. Our method shows better performance on testing accuracy compared with 1-bit Adam.

\begin{figure*}[ht]
 \centering
  \includegraphics[width=1.0\textwidth]{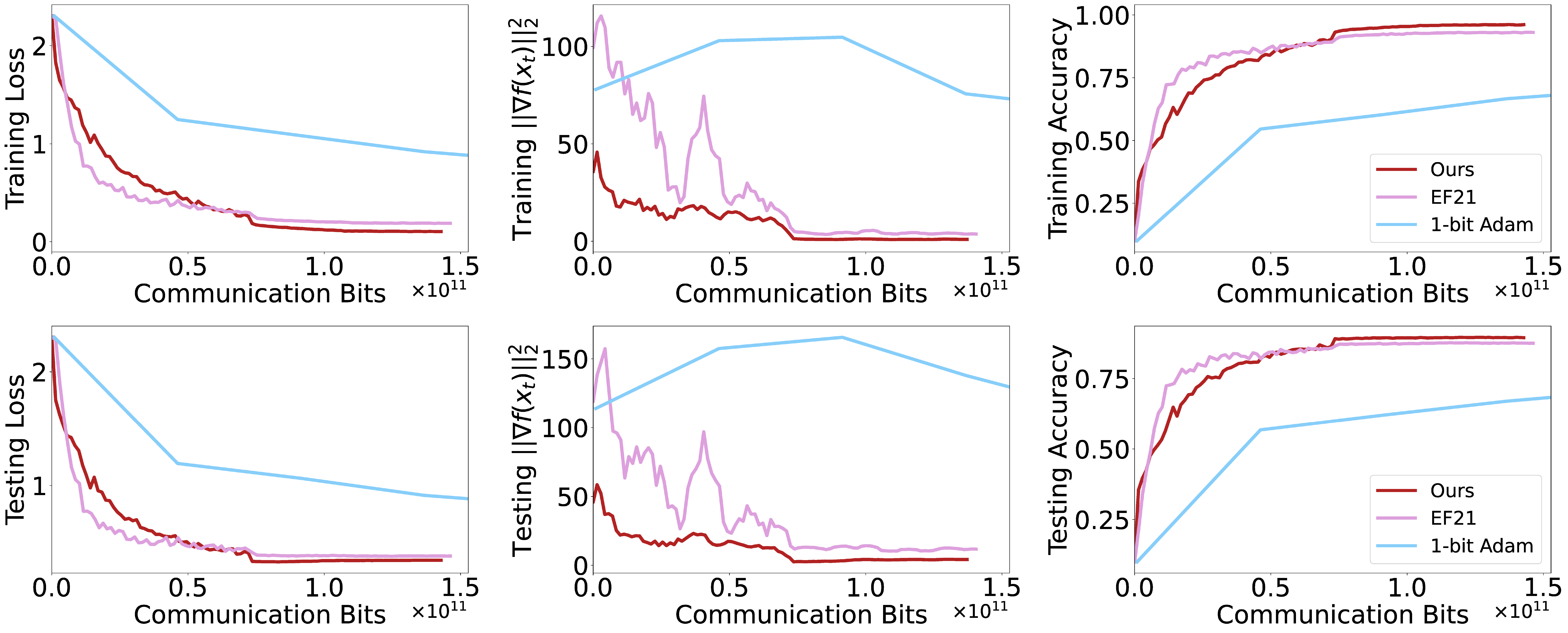}
 \caption{Comparison on training/testing loss, gradient norm and accuracy with respect to communication cost among our proposed method and the baselines when training \texttt{VGG-16} model on \texttt{CIFAR10}. 
}
 \label{fig:vgg_bits}
\end{figure*}

In Figure \ref{fig:vgg_bits}, the left two plots show that our proposed method obtains better training/testing loss than EF21 and 1-bit Adam, though EF21 intermediately shows a faster decrease in the loss function. From the middle two plots, we can observe that our proposed CD-Adam obtains smaller training/testing gradient norms than EF21 and 1-bit Adam under the same communication budget. The right two plots show that CD-Adam obtains better training/testing accuracy than EF21 and 1-bit Adam, though EF21 intermediately shows a faster increase in accuracy as well.

\begin{figure*}[ht]
 \centering
  \includegraphics[width=1.0\textwidth]{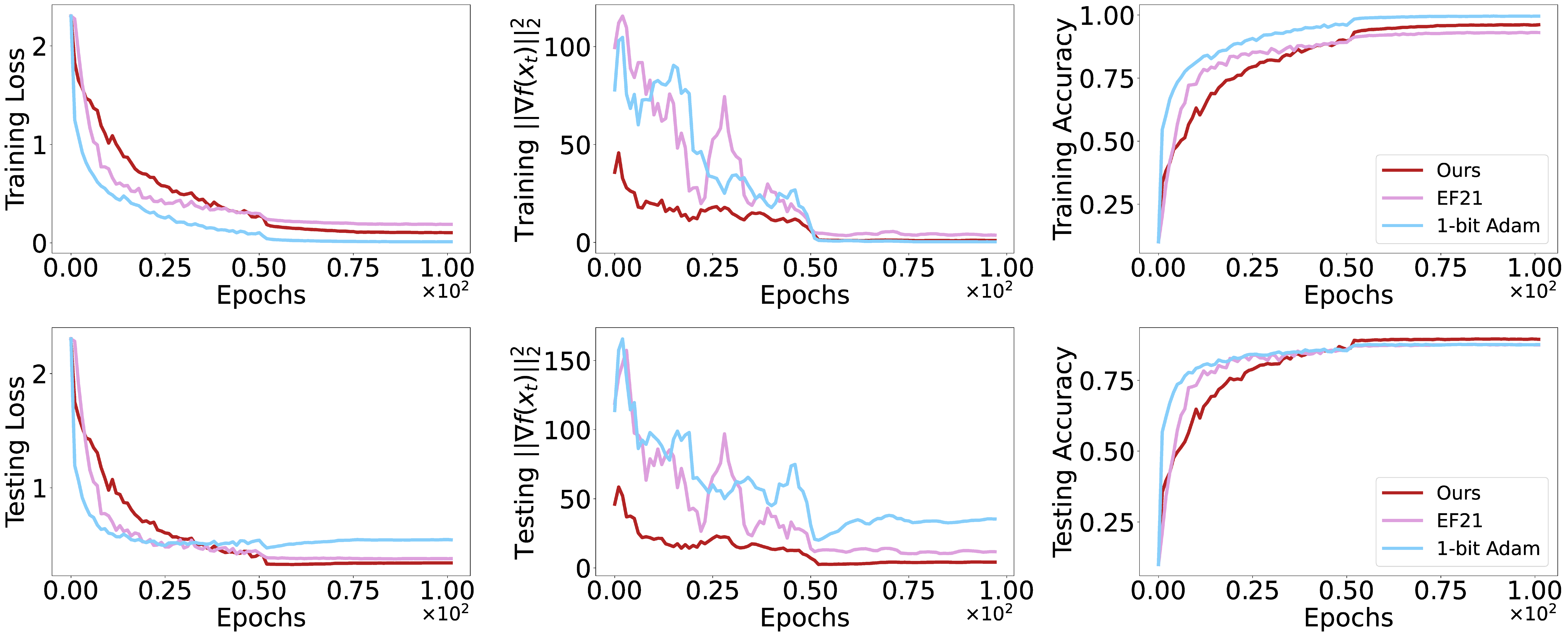}
 \caption{Comparison on training/testing loss, gradient norm and accuracy with respect to epochs among our proposed method and the baselines when training \texttt{VGG-16} model on \texttt{CIFAR10}.}
 \label{fig:vgg_epoch}
\end{figure*}

In Figure \ref{fig:vgg_epoch}, the left two plots show that our proposed method achieves better training/testing loss than EF21, though EF21 intermediately shows a faster decrease in the loss function. Our method shows better performance on testing loss compared with 1-bit Adam. The middle two plots in Figure \ref{fig:vgg_epoch} show that our proposed CD-Adam obtains a smaller gradient norm than EF21 and 1-bit Adam under the same epochs. The right two plots show that our proposed method achieves high accuracy on both the training and testing sets. It shows that our proposed method achieves better training/testing loss than EF21, though EF21 intermediately shows a faster increase in accuracy. Our CD-Adam also shows a better testing accuracy than 1-bit Adam.

\begin{figure*}[ht]
 \centering
  \includegraphics[width=1.0\textwidth]{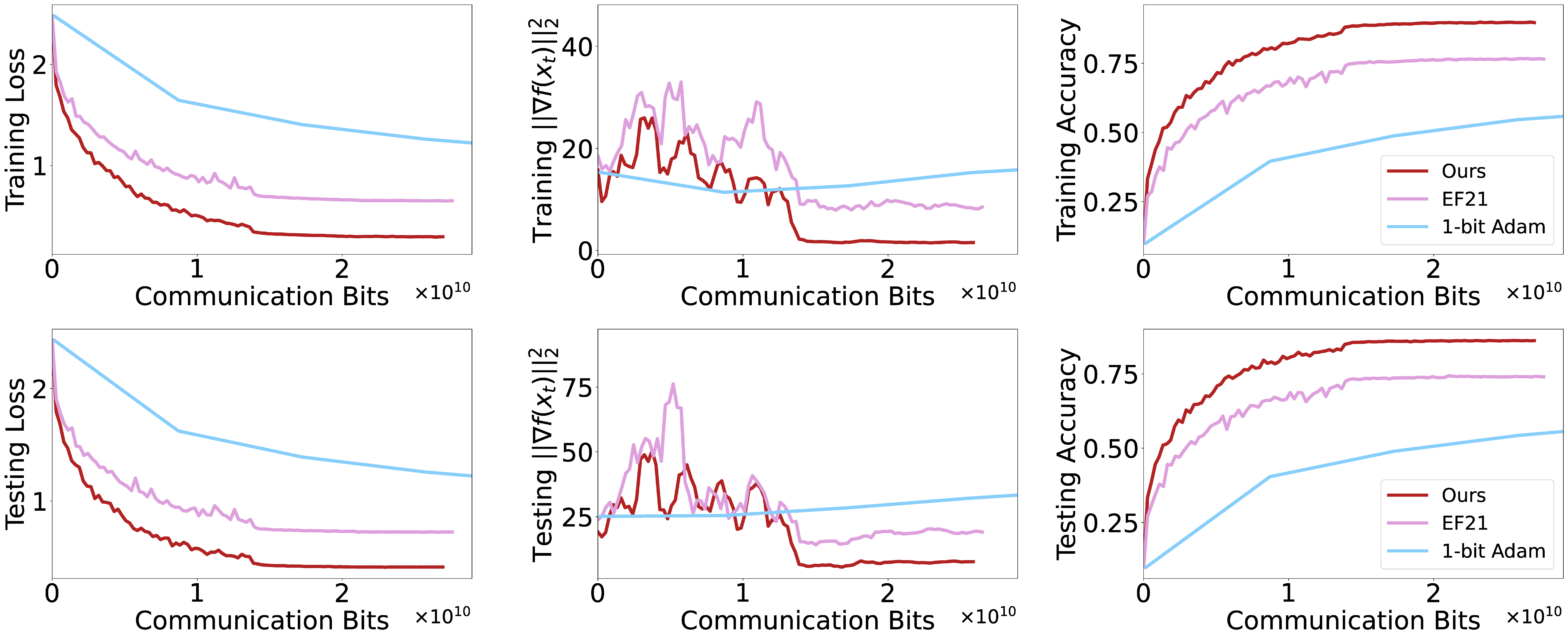}
 \caption{Comparison on training/testing loss, gradient norm and accuracy with respect to communication cost among our proposed method and the baselines when training \texttt{WideResNet-16-4} model on \texttt{CIFAR10}.}
 \label{fig:widern_bits}
\end{figure*}

In Figure \ref{fig:widern_bits}, the left two plots and the right two plots show that our proposed method maintains a significant advantage on training/testing loss and accuracy compared with EF21 and 1-bit Adam. The middle two plots show that our proposed CD-Adam achieves smaller gradient norms than EF21 and 1-bit Adam. In addition, 1-bit Adam initially shows a lower gradient norm while its gradient norm diverges later. 

\begin{figure*}[ht]
 \centering
  \includegraphics[width=1.0\textwidth]{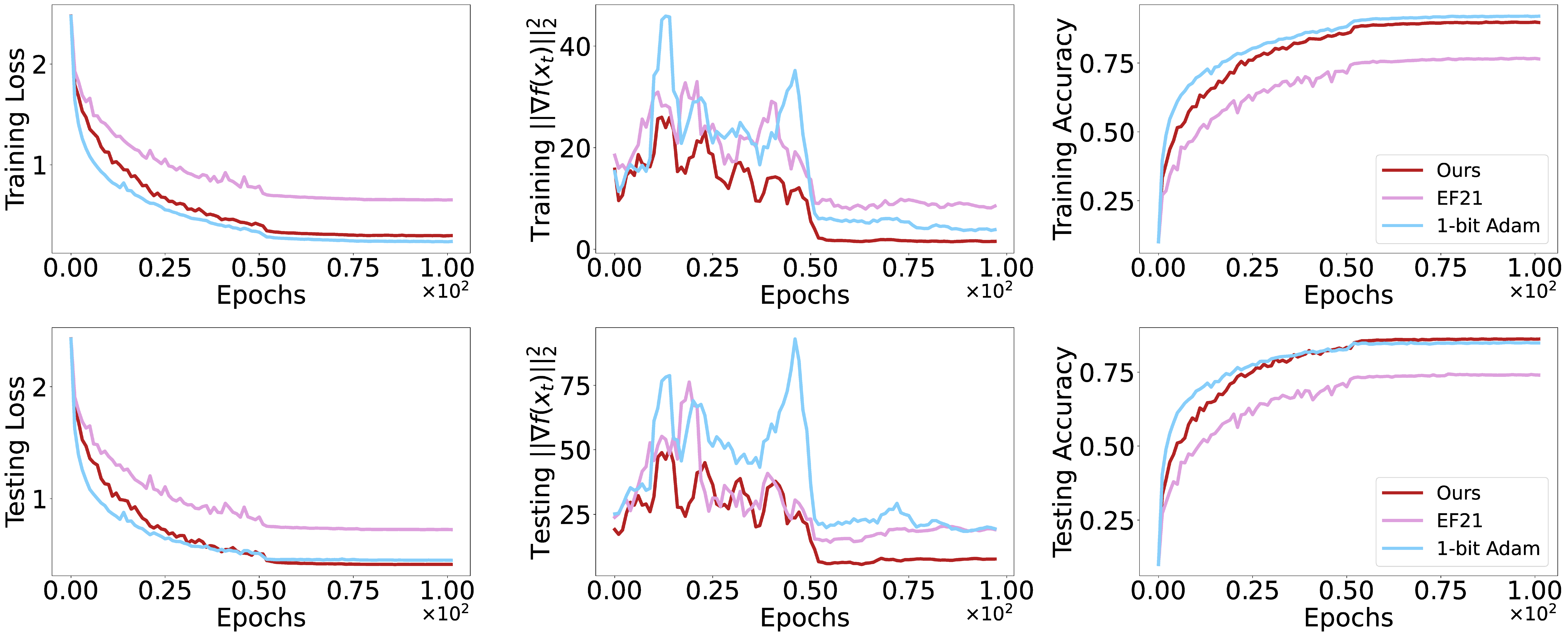}
 \caption{Comparison on training/testing loss, gradient norm and accuracy with respect to epochs among our proposed method and the baselines when training \texttt{WideResNet-16-4} model on \texttt{CIFAR10}.}
\label{fig:widern_epoch}
\end{figure*}

In Figure \ref{fig:widern_epoch}, the left two plots and right two plots show that our proposed method achieves better training/testing loss and accuracy than EF21. Furthermore, CD-Adam still achieves better performance on testing loss and testing accuracy than 1-bit Adam, while 1-bit Adam keeps a slighter better performance on training loss and training accuracy. The middle two plots in Figure \ref{fig:widern_bits} show that CD-Adam achieves better training/testing gradient norms compared with EF21 and 1-bit Adam. 

\subsection{Additional results}
\textbf{Average runtime and total bits: } The overhead on compression as well as the worker-side update is quite small, as shown from Table \ref{tab:runtime}. Note that EF21 takes more time since it adopts top-$k$ which is harder to compute compared to scaled sign compressor.

\begin{table}[ht]
    \centering
    \begin{tabular}{c|cc}
        \toprule
        Method & Avg Time (s/iter)  & Total Bits \\
        \midrule
        Uncompressed & 1.015 & $32d$$\times$$2T$ \\
        EF21 & 1.402 & $\approx(32k$$\times$$2)$$\times$$2T$ \\
        1-bit Adam &  1.041 & $32d$$\times$$2T_1$+$(32$+$d)$$\times$$2(T$-$T_1)$\\
        \textbf{CD-Adam} &  1.134 & $\mathbf{(32}$\textbf{+}$\mathbf{d)}$$\mathbf{\times}$$\mathbf{2T}$\\
        \bottomrule
    \end{tabular}
    \caption{Average runtime and total bits}
    \label{tab:runtime}
\end{table}

\textbf{Ablation on $n$ and $\tau$: } The left plot in Figure \ref{fig:n-tau} shows the effect of the total amount of workers $n$ on the convergence of training loss. It shows that a larger $n$ leads to a faster decrease on training loss but not absolutely leads to a better convergence rate. The right plot in Figure \ref{fig:n-tau} shows the effect of batch size $\tau$ on the convergence of training loss. We observe that a larger number of batch size achieves a faster convergence rate.
\begin{figure}[ht!]
    \centering
    \includegraphics[width=0.6\textwidth]{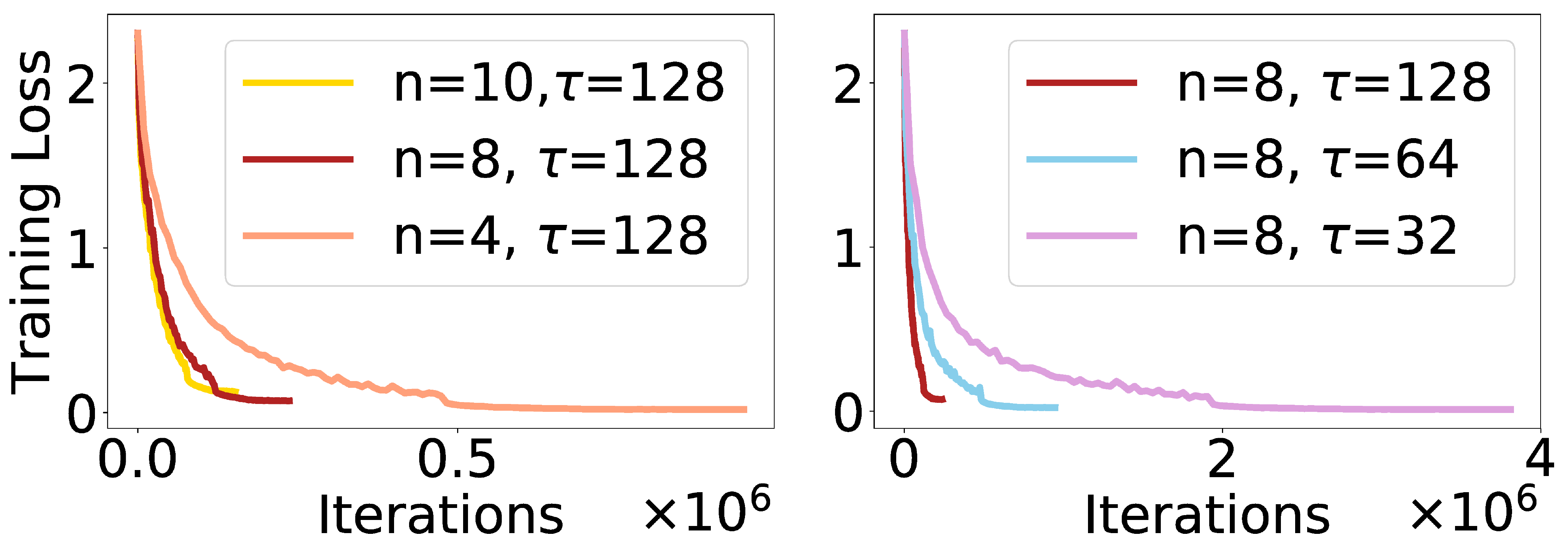}
    \caption{Ablation on $n$ and $\tau$}
    \label{fig:n-tau}
\end{figure}

\clearpage

\end{document}